\documentclass[12pt]{article}
\usepackage{amstext,amssymb,amsmath,epsf,graphics}
\newtheorem{assumption}{Assumption}
\newtheorem{defn}{Definition}
\newtheorem{thm}{Theorem}
\newtheorem{lemma}{Lemma}

\newcommand{\f}{f}

\newcommand{\g}{f^{-1}}

\renewcommand{\geq}{\geqslant}
\renewcommand{\leq}{\leqslant}

\newcommand{\bm}[1]{\boldsymbol #1}

\newcommand{\btheta}{\boldsymbol \theta}

\renewcommand{\Box}{\hspace*{\fill}~$\square$ \\}

\newcommand{\Real}{\mbox{$\mathbb{R}$}}

\begin{document}

\title{Gradient-Free Learning Based on the Kernel and the Range Space}
\author{Kar-Ann Toh$^{1}$\footnote{Corresponding author.}, Zhiping Lin$^{2}$, Zhengguo Li$^{3}$,
\\ Beomseok Oh$^{2}$ and Lei Sun$^{4}$\\
{\footnotesize $^1$School of Electrical and Electronic Engineering}\\*[-2.5mm]
{\footnotesize Yonsei University, Seoul, Korea\ 03722} \\
{\footnotesize $^2$School of Electrical and Electronic Engineering} \\*[-2.5mm]
{\footnotesize Nanyang Technological University, Singapore 639798} \\
{\footnotesize $^3$Institute for Infocomm Research} \\*[-2.5mm]
{\footnotesize 1  Fusionopolis Way, Singapore 138632} \\
{\footnotesize $^4$School of Information and Electronics} \\*[-2.5mm]
{\footnotesize Beijing Institute of Technology Beijing, PR China, 100081} \\
{\footnotesize Emails: katoh@yonsei.ac.kr, ezplin@ntu.edu.sg,
ezgli@i2r.a-star.edu.sg,}\\*[-2.5mm] {\footnotesize bsoh@ntu.edu.sg, and
sunlei@bit.edu.cn} }
\maketitle

\begin{abstract}
In this article, we show that solving the system of linear equations by manipulating the
kernel and the range space is equivalent to solving the problem of least squares error
approximation. This establishes the ground for a gradient-free learning search when the
system can be expressed in the form of a linear matrix equation. When the nonlinear
activation function is invertible, the learning problem of a fully-connected multilayer
feedforward neural network can be easily adapted for this novel learning framework. By a
series of kernel and range space manipulations, it turns out that such a network learning
boils down to solving a set of cross-coupling equations. By having the weights randomly
initialized, the equations can be decoupled and the network solution shows relatively
good learning capability for real world data sets of small to moderate dimensions. Based
on the structural information of the matrix equation, the network representation is found
to be dependent on the number of data samples and the output dimension.
\end{abstract} 
{\bf Keywords: } Least Squares Error, Linear Algebra, Multilayer Neural Networks.

\section{Introduction}

\subsection{Background}

The problem of supervised learning has often been treated as an optimization task where
an error metric is minimized. For problems which can be formulated as the system of
linear equations, an estimation based on the least squares error minimization provides an
analytical solution. Because the size of data samples and the size of parameters seldom
have an exact match, the solution comes in either the \emph{primal} form for
over-determined systems or the \emph{dual} form for under-determined systems.

For nonlinear formulations, a search based on the gradient descent is often utilized
towards seeking a numerical solution. In the early days of multilayer network learning, a
brilliant utilization of the chain-rule in the name of backprogagation algorithm
\cite{KelleyH1,Werbos10,HechNiel1,Werbos33,Werbos11,Haykin1} had made possible numerical
computation of gradients to arrive at reasonable learning solution for complex network
functions under the extremely limited memory and processing power during that time. This
algorithm, together with the advancement in personal computer and theoretical support on
approximate realization of continuous mappings \cite{Funa1,Hornik1,Cybenko1,HechNiel3}
had paved the way towards the booming of applications using neural networks in the 1980s.
Subsequently, in the 1990s, more computationally complex first- and second-order search
methods became realizable due to another leap of the computational hardware and memory
advancement (see e.g., \cite{Battiti1,Patrick11,Barnard1}). Apart from the exploration of
shallow \cite{Duin11} and thin \cite{Guliyev1} network architectures, attempts to seek a
globally optimized network learning can be found \cite{ShangY1,TangZY1}.

In the late 2000s, the artificial neural network reemerged in the name of deep learning
\cite{Ian1} due to the much improved training regimens, more powerful computer processors
and large amount of data. Moreover, the large number of high-level open source libraries
such as the Tensorflow, Keras, Pytorch, Caffe, Theano, Microsoft Cognitive Toolkit,
Apache MXNet, etc., have significantly elevated the popularization of deep learning
applications in various fields. Despite the much improved network learning regimens, yet
a large amount of deep learning falls back to the chain-rule based backprogagation when
the full functional gradient cannot be easily computed.

In search of more powerful learning methods for more complex networks, several
developments attempted beyond backpropagation and aimed at better computational
efficiency. A notable example is the \emph{alternating direction method of multipliers}
(ADMM) approach \cite{Taylor1} which combines the alternating direction method and the
Bregman iteration to train networks without the gradient descent steps. By forgoing the
global optimality attempt as a whole, this method reduces the network training problem to
a sequence of minimization substeps that can each be solved globally in closed form.
Together with the nonlinear search on the activation function, such sequential
minimization renders the overall minimization process iterative. Another example is given
in \cite{Wilamov1} where the network learning has been accomplished by utilizing the
forward-only computation instead of backpropagation. Some other works explored the
synthetic gradients towards pipeline parallelism  to gain computational efficiency
\cite{Jaderberg1}. From another perspective, the derivative-free optimization has also
been attempted. A review of derivative-free optimization algorithms on 22 implementations
for 502 problems can be found in \cite{Rios1}. Such an active research in the field shows
that complex functional and network learning remain an important topic of high impact for
advancement in the field.

\subsection{Motivations and Contributions}

When a common activation function is chosen in each layer, the fully-connected
feedforward neural network can be expressed compactly by matrix-vector notation despite
of its complex and nonlinear structure. This is particularly applicable to networks with
a replicative layer structure. Under such a structural resemblance to the system of
linear equations, it would be interesting to see if any analytical means from the
literature of linear systems can be found towards solving the network weights. This forms
our first motivation.

The next issue is related to the well accustomed strategy of error cost driven search for
learning the network weights. The error cost metric or the loss function comes in various
forms, ranging from the least squares distance, through the logistic cost, to the entropy
cost. While the logistic cost is apparently originated from the statistical viewpoint for
classification error counting, the entropy cost seems to arise from the information
theoretic viewpoint. Inherently, both of these cost formulations attempt to handle
outliers that the least squares error distance observes under limited sample size.
However, the complexity of the logistic and the entropy costs have often masked the
advantages that they carry, particularly in applications related to complex real-world
data. In view of the large data sets available in the recent context, the least squares
error cost remains a simple and yet effective approach to solving real-world problems
even though it may be far from fully satisfactory. However, the need for solving the
error gradient formulation according to the first-order necessary condition for
optimality remains an almost inevitable task. In order to seek for an analytical
solution, it is sensible to see whether the gradient of error computation can be
by-passed. Hence the second motivation.

This article addresses the issue of effective learning without needing of gradient
computation.\footnote{The preliminary ideas in this paper have been presented at the 17th
IEEE/ACIS International Conference on Computer and Information Science (ICIS 2018)
\cite{Toh97} where the first analytic learning network was introduced.} Unlike many
gradient-free search methods \cite{Rios1}, a surrogate of gradient is not needed in our
formulation. One of our main contributions is the discovery of the equivalence between
the manipulation under the kernel-and-range space and that under the least squares error
approximation based on the existing theory of linear algebra. This observation is
exploited in multiple layer network learning where it turns out that the network weights
can be expressed in analytical but cross-coupling form. By prefixing a set of the coupled
weights, a decoupled solution can be obtained. Our analysis on network representation
supports validity of the solution by the prefixed weight initialization. Moreover, it is
found that the minimum number of adjustable parameters needed depends only on the number
of data samples apart from the output dimension. The empirical evaluations using both
synthetic data and benchmark real-world data validated the feasibility and the efficiency
of the formulation.

\subsection{Organization}

The article is organized as follows. A revisit to the linear algebra in dealing with the
system of linear equations is given right after this introductory section. Essentially we
found that a manipulation of linear equations under the column and row spaces boils down
to having a least squares error approximation to fitting the equations. This observation
is next exploited to solve the network training problem in
Section~\ref{sec_net_learning}. Starting from a single-layer network, through the
two-layer network, and finally the multilayer networks, we show that network learning can
be performed without using error gradient descent. The network representation is further
shown to be dependent on the number of data samples apart from the output dimension. In
Section~\ref{sec_synthetic}, two case studies are performed to observe the learning
behavior. The feasibility of the formulation is further studied in
Section~\ref{sec_expts} on real-world data sets. Finally, some concluding remarks are
given in Section~\ref{sec_conclusion}.

\section{Preliminaries} \label{sec_prelim}

\subsection{Least Squares Approximation and Least Norm Solution}\label{sec_LS_approx}

Consider the system of linear equations given by
\begin{equation}\label{eqn_LinearEqn}
    {\bf A}\btheta = {\bf b},
\end{equation}
where $\btheta\in\Real^{d}$ is the unknown to be solved, ${\bf A}\in\Real^{m\times d}$
and ${\bf b}\in\Real^{m}$ are the known system data. For an over-determined system with
$m>d$, the $m$ equations in \eqref{eqn_LinearEqn} are generally unsolvable when a strict
equality for the equation is desired (see e.g., \cite{StrangG1}). However, by multiplying
${\bf A}^T$ to both sides of \eqref{eqn_LinearEqn}, the resultant $d$ equations
\begin{equation}\label{eqn_NormalEqn}
    {\bf A}^T{\bf A}\btheta = {\bf A}^T{\bf b}
\end{equation}
give the least squares error solution \eqref{eqn_d_space_primal} \cite{StrangG1} under
the \emph{primal solution space} as follows:
\begin{equation}\label{eqn_d_space_primal}
    \hat{\btheta} = ({\bf A}^T{\bf A})^{-1}{\bf A}^T{\bf b}.
\end{equation}
This can be understood by splitting ${\bf b}$ into the solvable part $\hat{{\bf b}}$,
which lies in the \emph{range} (\emph{column}) \emph{space} of ${\bf A}$; and the
unsolvable part $\bm{e}$, which lies in the \emph{null (row) space} of ${\bf A}$. The
null space is also known as the \emph{kernel} of ${\bf A}$. Based on this splitting,
\eqref{eqn_LinearEqn} can be written as
\begin{eqnarray}\label{eqn_LinearEqn3}
    {\bf A}\btheta &=& \hat{{\bf b}} + \bm{e}.
\end{eqnarray}
Since ${\bf A}\btheta$ is on the linear manifold of the Euclidean space, then based on
\eqref{eqn_LinearEqn} and \eqref{eqn_LinearEqn3}, we have
\begin{eqnarray}
    \|{\bf A}\btheta - {\bf b}\|^2 &=& \|{\bf A}\btheta - \hat{{\bf b}} - \bm{e} \|^2
    \nonumber\\
    &=&  \|\bm{e} \|^2 ,
    \label{eqn_squared_error_distance}
\end{eqnarray}
because ${\bf A}\btheta - \hat{{\bf b}}={\bf 0}$ as $\hat{{\bf b}}$ is the solvable part.
Since \eqref{eqn_NormalEqn} is the normal equation minimizing
\eqref{eqn_squared_error_distance}, we see that solving \eqref{eqn_LinearEqn} based on
\eqref{eqn_NormalEqn} minimizes the least squares error.

When the system is under-determined (i.e., $m<d$), then the number of equations is less
than that of the unknowns. However, we can project the unknown $\btheta\in\Real^d$ onto
the $\Real^m$ subspace \cite{Madych1} utilizing the row space of ${\bf A}$:
\begin{equation}\label{eqn_m_space}
    \btheta = {\bf A}^T\bm{s},
\end{equation}
where $\bm{s}\in\Real^m$. Substituting \eqref{eqn_m_space} into \eqref{eqn_LinearEqn}
gives
\begin{equation}\label{eqn_LinearEqn_m_space}
    {\bf A}{\bf A}^T\bm{s} = {\bf b}.
\end{equation}
Since ${\bf A}{\bf A}^T$ is now of full rank when ${\bf A}$ has full rank, $\bm{s}$ has a
unique solution given by
\begin{equation}\label{eqn_LinearEqn_m_space_s}
    \hat{\bm{s}} = ({\bf A}{\bf A}^T)^{-1}{\bf b}.
\end{equation}
Putting this $\bm{s}$ into \eqref{eqn_m_space} gives
\begin{equation}\label{eqn_m_space_dual}
    \hat{\btheta} = {\bf A}^T({\bf A}{\bf A}^T)^{-1}{\bf b}.
\end{equation}
This solution corresponds to minimizing $\|\btheta\|^2_2$ subject to ${\bf A}\btheta =
{\bf b}$ which is known as the \emph{least norm solution} \cite{Boyd1}. Interestingly,
this solution in the \emph{dual space} (the $m$ row space of ${\bf A}$) also minimizes
the sum of squared errors! This can be easily seen by substituting \eqref{eqn_m_space}
and \eqref{eqn_LinearEqn_m_space_s} into \eqref{eqn_squared_error_distance} which gives a
minimized $({\bf A}\hat{\btheta} - {\bf b})^T({\bf A}\hat{\btheta} - {\bf b})$. With
these observations, we make the following observation based on the conventional ground of
linear algebra (see e.g., \cite{Albert1,Adi1,SLCampbell1,Boyd1}).

\begin{lemma} \label{lemma_LS}
Solving for $\btheta$ in the system of linear equations of the form \eqref{eqn_LinearEqn}
in the column space (range) of ${\bf A}$ or in the row space (kernel) of ${\bf A}$ is
equivalent to solving the least squares error of the form
\eqref{eqn_squared_error_distance}. Moreover, the resultant solution $\hat{\btheta}$ is
unique with a minimum-norm value in the sense that
$\|\hat{\btheta}\|^2_2\leq\|\btheta\|^2_2$ for all feasible $\btheta$.
\end{lemma}

\begin{proof}
The first assertion has been shown in the derivations above. The unique solution is given
by either \eqref{eqn_d_space_primal} or \eqref{eqn_m_space_dual} depending on the
available space provided by ${\bf A}$. The remaining task is to show the minimum-norm
value of the solution. The proof for the under-determined system ($m<d$) follows that of
\cite{Boyd1} where we first suppose ${\bf A}\btheta = {\bf b}$ in \eqref{eqn_LinearEqn}
and arrive at ${\bf A}(\btheta-\hat{\btheta})=\textbf{0}$ which can be substituted into
\begin{eqnarray}
  (\btheta-\hat{\btheta})^T\hat{\btheta} &=& (\btheta-\hat{\btheta})^T{\bf A}^T({\bf A}{\bf A}^T)^{-1}{\bf b} \nonumber\\
    &=& ({\bf A}(\btheta-\hat{\btheta}))^T({\bf A}{\bf A}^T)^{-1}{\bf b} \nonumber\\
    &=& 0.
\end{eqnarray}
This means that $(\btheta-\hat{\btheta})\perp\hat{\btheta}$ and so
\begin{eqnarray}
  \|\btheta\|^2_2 &=& \|\hat{\btheta}+\btheta-\hat{\btheta}\|^2_2 \nonumber\\
    &=& \|\hat{\btheta}\|^2_2 + \|\btheta-\hat{\btheta}\|^2_2 \nonumber\\
    &\geq& \|\hat{\btheta}\|^2_2,
\end{eqnarray}
i.e., $\hat{\btheta}$ has the smallest norm among all feasible solutions.

On the other hand, for the over-determined system ($m>d$), the proof follows that in
\cite{Albert1} where we have $\hat{{\bf b}}-{\bf A}\btheta$ in the range space of ${\bf
A}$ and $\bm{e}\perp\hat{{\bf b}}-{\bf A}\btheta$. Therefore
\begin{eqnarray}
  \|{\bf b}-{\bf A}\btheta\|^2_2 &=& \|\hat{{\bf b}}-{\bf A}\btheta+\bm{e}\|^2_2 \nonumber\\
  &=& \|\hat{{\bf b}}-{\bf A}\btheta\|^2_2 + \|\bm{e}\|^2_2 \nonumber\\
  &\geq& \|\bm{e}\|^2_2
\end{eqnarray}
This lower bound is attained at $\hat{\btheta}$ only if $\hat{\btheta}$ is such that
${\bf A}\hat{\btheta}=\hat{{\bf b}}$. Similarly, any such $\hat{\btheta}$ can be
decomposed into two orthogonal vectors $\hat{\btheta}^r+\hat{\btheta}^k$ where
$\hat{\btheta}^r$ falls in the range space of ${\bf A}^T$ and $\hat{\btheta}^k$ falls in
the kernel of ${\bf A}$. Thus ${\bf A}\hat{\btheta}={\bf A}\hat{\btheta}^r$ so that
$\|{\bf b}-{\bf A}\hat{\btheta}\|^2_2=\|{\bf b}-{\bf A}\hat{\btheta}^r\|^2_2$ and
$\|\hat{\btheta}\|^2_2 = \|\hat{\btheta}^r\|^2_2 + \|\hat{\btheta}^k\|^2_2
\geq\|\hat{\btheta}^r\|^2_2$. This completes the proof.
\end{proof}

The above result can be generalized to the system of equations with multiple outputs as
follows.

\begin{lemma} \label{lemma_LS_matrix}
Solving for $\bm{\Theta}$ in the system of linear equations of the form
\begin{equation}\label{eqn_LinearEqn_matrix}
    {\bf A}\bm{\Theta} = {\bf B}, \ \ \ {\bf A}\in\Real^{m\times d},\ \bm{\Theta}\in\Real^{d\times q},
    \ {\bf B}\in\Real^{m\times q}
\end{equation}
in the column space (range) of ${\bf A}$ or in the row space (kernel) of ${\bf A}$ is
equivalent to minimizing the sum of squared errors ($\textup{SSE}$) given by
\begin{equation}\label{eqn_squared_error_distance_trace}
    \textup{SSE} = trace\left( ({\bf A}\bm{\Theta}-{\bf B})^T({\bf A}\bm{\Theta}-{\bf B}) \right).
\end{equation}
Moreover, the resultant solution $\hat{\bm{\Theta}}$ is unique with a minimum-norm value
in the sense that $\|\hat{\bm{\Theta}}\|^2_2\leq\|\bm{\Theta}\|^2_2$ for all feasible
$\bm{\Theta}$.
\end{lemma}
\begin{proof}
Equation \eqref{eqn_LinearEqn_matrix} can be re-written as a set of multiple linear
systems of that in \eqref{eqn_LinearEqn} as
\begin{equation}\label{eqn_LinearEqn_matrix_decomposed}
    {\bf A}[\btheta_1,\cdots,\btheta_q] = [{\bf b}_1,\cdots,{\bf b}_q].
\end{equation}
Since the trace of $({\bf A}\bm{\Theta}-{\bf B})^T({\bf A}\bm{\Theta}-{\bf B})$ is equal
to the sum of the squared lengths of the error vectors ${\bf A}\btheta_i-{\bf b}_i$,
$i=1,2,...,q$, the unique solution $\hat{\bm{\Theta}}=({\bf A}^T{\bf A})^{-1}{\bf
A}^T{\bf B}$ in the range space of ${\bf A}$ or that $\hat{\bm{\Theta}}={\bf A}^T({\bf
A}{\bf A}^T)^{-1}{\bf B}$ in the kernel of ${\bf A}$, not only minimizes this sum, but
also minimizes each term in the sum \cite{Duda1}. Moreover, since the column and the row
spaces are independent, the sum of the individually minimized norms is also minimum.
\end{proof}

In summary, a pre-multiplication of ${\bf A}^T$ to both sides of an over-determined
system in \eqref{eqn_LinearEqn} or \eqref{eqn_LinearEqn_matrix} implies a least squares
error minimization process in the primal (column) space of ${\bf A}$ ($\Real^d$).
Moreover, solving the under-determined system in the dual (row) space of ${\bf A}$
($\Real^m$) boils down to seeking a least squares error solution as well. The
relationship between \eqref{eqn_d_space_primal} and \eqref{eqn_m_space_dual} can be
linked by applying the Searle Set of identities (see e.g., \cite{Petersen1}) and the
existence of the limiting case \cite{Albert1} for which $\lim_{\lambda\rightarrow
0}(\lambda{\bf I}+{\bf A}{\bf B})^{-1}{\bf A}=\lim_{\lambda\rightarrow 0}{\bf
A}(\lambda{\bf I}+{\bf B}{\bf A})^{-1}$. Collectively, the inverse of matrix ${\bf A}$
can be written as ${\bf A}^{\dag}$ where in general, such matrix inversion which
satisfies (i) ${\bf A}{\bf A}^{\dag}{\bf A}={\bf A}$, (ii) ${\bf A}^{\dag}{\bf A}{\bf
A}^{\dag}={\bf A}^{\dag}$, (iii) $({\bf A}{\bf A}^{\dag})^*={\bf A}{\bf A}^{\dag}$ and
(iv) $({\bf A}^{\dag}{\bf A})^*={\bf A}^{\dag}{\bf A}$ is known as the
\emph{pseudoinverse} or the \emph{Moore-Penrose inverse}
\cite{Moore1,Bjerhammar1,Penrose1}.\footnote{The pesudoinverse of a rectangular matrix
was first introduced by Moore in \cite{Moore1} and later rediscovered independently by
Bjerhammar \cite{Bjerhammar1} and Penrose \cite{Penrose1}. The superscript symbol $^*$
denotes the conjugate transpose.} The minimum norm property of the solution provides
capability for good learning generalization. This is called implicit
regularization in \cite{ZhangCY1}.\\

This process of solving the algebraic equations under the \emph{kernel-and-range} (KAR)
space based on the \emph{Moore-Penrose inverse} with implicit least squares error seeking
will be exploited to solve the network learning problem in
Section~\ref{sec_net_learning}.


\subsection{Invertible Function}

In this development, we need to invert the network over the activation function for
solution seeking. Such an inversion is performed through function inversion. The inverse
function is defined as follows.

\begin{defn}
Consider a function $f$ which maps $x\in\Real$ to $y\in\Real$, i.e., $y=f(x)$. Then the
inverse function for $f$ is such that $f^{-1}(y)=x$.
\end{defn}

A good example for invertible function is the \emph{logit-sigmoid} pair on the domain
$x\in[0,1)$, which shall be adopted in our experiments.

\subsection{Fully-connected Feedforward Network (FFN)}

In our context, each layer (say, the $k^{th}$-\emph{layer} where $k\in\{1,\cdots,n\}$) of
a fully-connected feedforward neural network is referred to as composing of a multiplying
weight matrix $\bm{W}_k$ and an activation function $\f_k$ which takes in the matrix
product. Fig.~\ref{fig_DN4Layer2} shows a fully-connected feedforward network of $n$
layers. In this representation, the bias term in each layer has been isolated from the
inputs of current layer (or the outputs of the previous layer) where the weight matrix
can be partitioned as $\bm{W}_k=\left[\begin{array}{c} {\bf w}_k^T \\ \textsf{W}_k \\
\end{array}\right]$ with ${\bf w}_k=[w_1,\cdots,w_{h_{k}}]^T\in\Real^{h_{k}}$ being the
weights corresponding to the bias of the previous layer, and
$\textsf{W}_k\in\Real^{h_{k-1}\times h_k}$ being the weights for $h_{k}$ number of hidden
nodes at the $k^{th}$ layer. For simplicity, we assume $f=f_1=\cdots=f_n$ for all hidden
nodes.

\begin{figure}[hhh]
  \begin{center}
  \epsfxsize=12.88cm
  \epsffile[150   160  920   560]{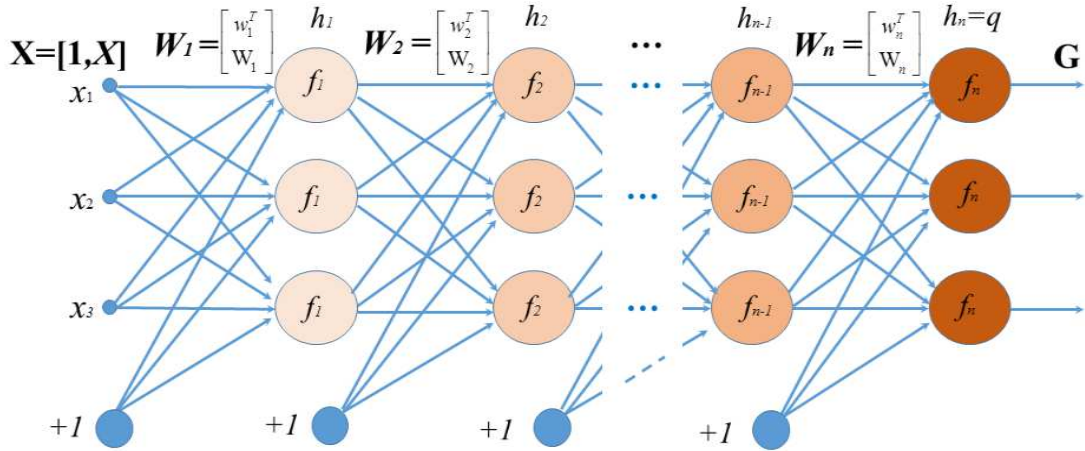}
  \caption{A fully-connected $n$-layer feedforward network. }
  \label{fig_DN4Layer2}
  \end{center}
\end{figure}

\clearpage

\section{Network Learning by Solving the System of Linear Equations} \label{sec_net_learning}

\subsection{Single-Layer Network}

Consider a single-layer FFN given by
\begin{equation}\label{eqn_1layer_net}
    {\bf G} = \f_1({\bf X}\bm{W}_1),
\end{equation}
where $\f_1(\cdot)$ is an activation function\footnote{Although $f=f_1=\cdots=f_n$, we
shall keep the subscript of the activation function to indicate the respective layer for
clarity purpose.} which operates elementwise on its matrix domain, ${\bf
X}\in\Real^{m\times (d+1)}$ is an augmented input (which includes the bias term),
$\bm{W}_1\in\Real^{(d+1)\times h_1}$ is the weight matrix to be estimated, and ${\bf
G}\in\Real^{m\times q}$ is the network output.

Let ${\bf Y}\in\Real^{m\times q}$ be the target output to be learned. Mathematically, the
learning problem is formulated as putting ${\bf G}={\bf Y}$ and then solve for the
adjustable parameters within the model ${\bf G}$. Suppose $\f_1$ satisfies the assumption
given by:
\begin{assumption}\label{assump_11}
  The function $\f_1$ operates elementwise on matrices so that\\
  $\f_1({\bf X}\bm{W}_1)={\bf Y}$ implies that
  there exists a full rank matrix ${\bf A}_1$ such that \\
  ${\bf A}_1\textbf{\textsf{W}}_1=\phi_1({\bf Y})$
  where $\textbf{\textsf{W}}_1$ is an elementwise transformed version of $\bm{W}_1$ and $\phi_1$ is a monotonic and injective function that operates elementwise
  on its matrix domain.
\end{assumption}
This means that if the function $\f_1$ is invertible, then $ \phi_1(\f_1({\bf
X}\bm{W}_1)) = \phi_1({\bf Y}) $ implies that ${\bf X}\bm{W}_1=\g_1({\bf Y})$ where ${\bf
A}_1={\bf X}$ and $\textbf{\textsf{W}}_1=\bm{W}_1$. If $\f_1$ is not invertible, then
there exists a transformation $\phi_1:\Real\mapsto\Real_+$ \footnote{$\Real$ denotes real
numbers, and $\Real_+$ denotes positive real numbers} such that $\phi_1(\f_1({\bf
X}\bm{W}_1))={\bf A}_1\textbf{\textsf{W}}_1$ where ${\bf A}_1=\phi_1(\f_1({\bf X}))$ and
$\textbf{\textsf{W}}_1=\phi_1(\f_1({\bf X}\bm{W}_1))/\phi_1(\f_1({\bf X}))$ (assuming
elementwise operation of the division). Since $\phi_1$ is a monotonic and injective
function, the relative numerical order of elements within ${\bf X}$ and ${\bf Y}$ is
unaltered. This means that the mapping between the data and the target can be fully
recovered when $\phi_1$ is known (a good example is $\phi_1(\cdot)=\exp(\cdot)$). With
this assumption on the property of $f_1$, the following result is obtained.
\begin{thm}
Suppose $\f_1$ satisfies Assumption~\ref{assump_11}. Then learning of the weights in the
single-layer network of \eqref{eqn_1layer_net} admits the solution
\begin{equation}\label{eqn_1layer_soln}
    {\textbf{\textsf{W}}_1} = {\bf A}_1^{\dag}\phi_1({\bf Y})
\end{equation}
in the sense of least squares error approximation to the target ${\bf Y}$ where ${\bf
A}_1^{\dag} = ({\bf A}_1^T{\bf A}_1)^{-1}{\bf A}_1^T$ when ${\bf A}_1^T{\bf A}_1$ has
full rank, or ${\bf A}_1^{\dag} = {\bf A}_1^T({\bf A}_1{\bf A}_1^T)^{-1}$ when ${\bf
A}_1{\bf A}_1^T$ has full rank.
\end{thm}

\begin{proof}
For an over-determined system learning a target output ${\bf Y}\in\Real^{m\times q}$, we
can put ${\bf G}={\bf Y}$ and learn the system by manipulating the kernel and the range
space using
\begin{equation}  \label{eqn_1layer_net_soln}
    \phi_1({\bf Y}) = {\bf A}_1\textbf{\textsf{W}}_1,
\end{equation}
where the transformed weight matrix can be solved analytically as follows:
\begin{eqnarray} \label{eqn_net_primal_soln}
     {\bf A}_1^T \phi_1({\bf Y}) &=& {\bf A}_1^T{\bf A}_1{\textbf{\textsf{W}}}_1 \nonumber\\
    \Rightarrow\hspace{5mm} ({\bf A}_1^T{\bf A}_1)^{-1}{\bf A}_1^T \phi_1({\bf Y}) &=& {\textbf{\textsf{W}}_1} .
\end{eqnarray}
This is the well-known solution to least squares error approximation.

For under-determined systems, the problem can be solved by putting
\begin{equation}\label{eqn_net_dual_soln_pre}
    {\textbf{\textsf{W}}}_1={\bf A}_1^T{\bf V}
\end{equation}
where ${\bf V}\in\Real^{m\times q}$ is obtained from
\begin{eqnarray}
    \phi_1({\bf Y}) &=& {\bf A}_1{\bf A}_1^T{\bf V} \nonumber\\
    \Rightarrow\hspace{5mm} ({\bf A}_1{\bf A}_1^T)^{-1} \phi_1({\bf Y}) &=& {\bf V} .
    \label{eqn_1layer_net_soln_dual}
\end{eqnarray}
Putting \eqref{eqn_1layer_net_soln_dual} into \eqref{eqn_net_dual_soln_pre} gives
\begin{equation}\label{eqn_net_dual_soln}
    {\textbf{\textsf{W}}_1}={\bf A}_1^T({\bf A}_1{\bf A}_1^T)^{-1} \phi_1({\bf Y}).
\end{equation}
A direct substitution of \eqref{eqn_net_dual_soln} into the \emph{trace} of $(\phi_1({\bf
Y}) - {\bf A}_1{\textbf{\textsf{W}}}_1)^T(\phi_1({\bf Y}) - {\bf
A}_1{\textbf{\textsf{W}}}_1)$ shows that the solution also minimizes the sum of squared
errors. This completes the proof.
\end{proof}

\noindent{\bf Remark 1: } Collectively, ${\bf A}_1^{\dag}$, which is obtained by either
\eqref{eqn_net_primal_soln} or \eqref{eqn_net_dual_soln}, can be taken as the
Moore-Penrose inverse of ${\bf A}_1$ \cite{Albert1}. The transformation of ${\bf
Y}=\f_1({\bf X}\bm{W}_1)$ into \eqref{eqn_1layer_net_soln} makes the system of nonlinear
equations into a linear one. This shows that learning of an FFN of single connection
layer \eqref{eqn_1layer_net} boils down to linear least squares estimation. Since both
$\textbf{\textsf{W}}_1$ (transformed by $\phi_1$) and $\bm{W}_1$ (the non-transformed
one) refer to the same set of weights of the network, we shall not distinguish between
them and use only $\bm{W}_k$, $k=1,\cdots,n$ for subsequent development. \Box

\subsection{Two-Layer Network}

Next, consider a two-layer network given by
\begin{equation}\label{eqn_2layer_net}
    {\bf G} = \f_2([{\bf 1},\f_1({\bf X}\bm{W}_1)]\bm{W}_2),
\end{equation}
where $\f_1$ and $\f_2$ are activation functions which operate elementwise upon their
matrix domain, ${\bf X}\in\Real^{m\times (d+1)}$ is an augmented input,
$\bm{W}_1\in\Real^{(d+1)\times h_1}$ and $\bm{W}_2\in\Real^{(h_1+1)\times q}$ are the
unknown network weights,  ${\bf 1}=[1,...,1]^T\in\Real^{m\times 1}$ is the bias vector,
and ${\bf G}\in\Real^{m\times q}$ denotes the network output. The target to be learned is
${\bf Y}\in\Real^{m\times q}$.

The following assumption is needed for our learning formulation when the network output
${\bf G}=\f_2([{\bf 1},\f_1({\bf X}\bm{W}_1)]\bm{W}_2)$ is equated to the learning target
${\bf Y}$ in order to solve for the weight parameters.
\begin{assumption}\label{assump_21}
  The functions $\f_1$ and $\f_2$ operate elementwise on matrices so that
  $\f_2([{\bf 1},\f_1({\bf X}\bm{W}_1)]\bm{W}_2)={\bf Y}$ implies that
  for known $\bm{W}_1$,
  there exists a full rank matrix ${\bf A}_2$
  such that ${\bf A}_2\bm{W}_2=\phi_2({\bf Y})$
  where ${\bf A}_2=[{\bf 1},\f_1({\bf X}\bm{W}_1)]$
  and $\phi_2$ is a monotonic and injective function that operates elementwise on
  its matrix domain.
\end{assumption}
This means that $\f_2$ can either be inverted (which results in
$\phi_2(\cdot)=\g_2(\cdot)$) or transformed (in case of non-invertible) giving $[{\bf
1},\f_1({\bf X}\bm{W}_1)]\bm{W}_2=\phi_2({\bf Y})$ where $f_1$ provides sufficient rank
conditioning regarding the matrix $[{\bf 1},\f_1({\bf X}\bm{W}_1)]$ for solving the
linear equation using either the kernel or the range space. By separately considering the
weights corresponding to the output bias term, we can partition
$\bm{W}_2=\left[\begin{array}{c}
{\bf w}_2^T \\ \textsf{W}_2 \\
\end{array}\right]$ where ${\bf w}_2\in\Real^{q\times 1}$ and $\textsf{W}_2\in\Real^{h_1\times q}$.
Based on these partitioned weights, we have
\begin{eqnarray}
  \phi_2({\bf Y})
        &=& {\bf 1}\cdot{\bf w}_2^T + \f_1({\bf X}\bm{W}_1)\textsf{W}_2 \nonumber\\
  \Rightarrow\hspace{5mm}
  \left[\phi_2({\bf Y}) - {\bf 1}\cdot{\bf w}_2^T\right] \textsf{W}_2^{\dag}
        &=& \f_1({\bf X}\bm{W}_1), \label{eqn_2layer_net_dual}
\end{eqnarray}
where the symbol $\textsf{W}_2^{\dag}$ denotes the Moore-Penrose inverse of the weights
without bias, $\textsf{W}_2$.

Consider Assumption~\ref{assump_11} and suppose $\bm{W}_2$ is known, then the hidden
layer weights $\bm{W}_1$ can be solved by another inversion or transformation of equation
\eqref{eqn_2layer_net_dual} with respect to $f_1$:
\begin{eqnarray}
  \phi_1\left( \left[\phi_2({\bf Y}) - {\bf 1}\cdot{\bf w}_2^T\right] \textsf{W}_2^{\dag} \right) &=&  {\bf X}\bm{W}_1 \nonumber\\
  \Rightarrow\hspace{5mm}
  {\bf X}^{\dag}\phi_1\left( \left[\phi_2({\bf Y}) - {\bf 1}\cdot{\bf w}_2^T\right] \textsf{W}_2^{\dag} \right)
  &=&  {\bm{W}_1} .
  \label{eqn_2layer_net_W}
\end{eqnarray}

After having $\bm{W}_1$ computed, it can be used to re-compute $\bm{W}_2$ as follows:
\begin{eqnarray}
    \phi_2({\bf Y}) &=& [{\bf 1},\f_1({\bf X}{\bm{W}_1})]\bm{W}_2 \nonumber\\
    \Rightarrow\hspace{5mm}
    \left[{\bf 1},\f_1({\bf X}{\bm{W}_1})\right]^{\dag} \phi_2({\bf Y})
    &=& {\bm{W}_2} \label{eqn_2layer_net_fulltheta}
\end{eqnarray}
where $[{\bf 1},\f({\bf X}{\bm{W}_1})]^{\dag}$ denotes the Moore-Penrose
inverse of $[{\bf 1},\f({\bf X}{\bm{W}_1})]$. \\



\begin{thm} \label{thm_double}
Suppose $f_1$ and $f_2$ satisfy Assumptions~\ref{assump_11}--\ref{assump_21}. Then,
learning of the fully-connected two-layer network \eqref{eqn_2layer_net} in terms of the
least squares error approximation to target ${\bf Y}\in\Real^{m\times q}$ boils down to
solving the cross-coupling equations \textup{
\begin{eqnarray}
  \bm{W}_1 &=& {\bf X}^{\dag}\phi_1\left(\left[\phi_2({\bf Y}) - {\bf 1}\cdot{\bf w}_2^T\right]\textsf{W}_2^{\dag} \right),
  \label{eqn_soln_2layer_W}\\
   \textrm{and}\ \ \bm{W}_2 =\left[\begin{array}{c} {\bf w}_2^T \\ {\textsf{W}_2} \\
\end{array}\right] &=&
        \left[{\bf 1},\f_1({\bf X}\bm{W}_1)\right]^{\dag} \phi_2({\bf Y}).
    \label{eqn_soln_2layer_Theta}
\end{eqnarray}}
\end{thm}
\begin{proof}
Equation \eqref{eqn_soln_2layer_W} is obtained following the derivation steps in
\eqref{eqn_2layer_net_dual} through \eqref{eqn_2layer_net_W}. Equation
\eqref{eqn_soln_2layer_Theta} has been derived in \eqref{eqn_2layer_net_fulltheta}.
\end{proof}

\noindent{\bf Remark 2: } This result shows that the two-layer network learning problem
boils down to solving a system of equations with cross-coupling weights, i.e., $\bm{W}_1$
and $\bm{W}_2$ being inter-dependent. This implies that the system is recursive unless
one of the two weight matrices is prefixed. We shall observe in
Section~\ref{sec_representation} regarding the network representation capability when the
hidden weights are prefixed randomly. \Box

\subsection{Three-Layer Network and Beyond}

Consider a three-layer network given by
\begin{equation}\label{eqn_3layer_net}
    {\bf G} = \f_3(\left[{\bf 1},\f_2([{\bf 1},\f_1({\bf X}\bm{W}_1)]\bm{W}_2)\right]\bm{W}_3)
\end{equation}
where ${\bf X}\in\Real^{m\times (d+1)}$, $\bm{W}_1\in\Real^{(d+1)\times h_1}$,
$\bm{W}_2\in\Real^{(h_1+1)\times h_2}$, $\bm{W}_3\in\Real^{(h_2+1)\times q}$, ${\bf
1}=[1,...,1]^T\in\Real^{m\times 1}$, and output ${\bf G}\in\Real^{m\times q}$ with target
${\bf Y}\in\Real^{m\times q}$. Following the assumption in the two-layer case, we have
\begin{assumption}\label{assump_31}
  The functions $\f_i$, $i=1,2,3$ operate elementwise on matrices so that
  $\f_3(\left[{\bf 1},\f_2([{\bf 1},\f_1({\bf X}\bm{W}_1)]\bm{W}_2)\right]\bm{W}_3)={\bf Y}$
  implies that for known $\bm{W}_{j}$, $j=1,\cdots,k-1$,
  there is a full rank matrix ${\bf A}_k$ such that ${\bf A}_k\bm{W}_k={\bf B}_k$
  for each $k=2,3$ where\\
  \indent ${\bf A}_2=[{\bf 1},\f_1({\bf X}\bm{W}_1)]$, ${\bf A}_3=[{\bf 1},
  \f_2([{\bf 1},\f_1({\bf X}\bm{W}_1)]){\bm{W}_2}]$,\\
  \indent ${\bf B}_2=\phi_2(\phi_3({\bf Y}) \bm{W}_3^{\dag})$, ${\bf B}_3=\phi_3({\bf Y})$,\\
  and $\phi_k$, $k=2,3$ are monotonic and injective functions
  that operate elementwise on their matrix domain.
\end{assumption}
This implies that $\f_1$ in ${\bf A}_2$ and $\f_2$ in ${\bf A}_3$ are functions that
provide rank sufficiency for solving the system of linear equations. Then, similar to the
two-layer network, the weights corresponding to the output bias term can be separated out
as
$\bm{W}_3=\left[\begin{array}{c} {\bf w}_3^T \\ \textsf{W}_3 \\
\end{array}\right]$ where ${\bf w}_3\in\Real^{q\times 1}$ and $\textsf{W}_3\in\Real^{h_2\times q}$.
By replacing ${\bf G}$ with ${\bf Y}$ and taking the transformation or inverse function
of $\f_3$ (i.e., $\phi_3$) to both sides of \eqref{eqn_3layer_net} after moving the bias
weights over to the left hand side, we have
\begin{eqnarray}
  \phi_3({\bf Y}) - {\bf 1}\cdot{\bf w}_3^T &=& \f_2([{\bf 1},\f_1({\bf X}\bm{W}_1)]\cdot\bm{W}_2)\textsf{W}_3
  \nonumber\\ \Rightarrow\hspace{5mm}
  \left[\phi_3({\bf Y}) - {\bf 1}\cdot{\bf w}_3^T\right]\textsf{W}_3^{\dag} &=&
      \f_2([{\bf 1},\f_1({\bf X}\bm{W}_1)]\cdot\bm{W}_2) \nonumber\\
      \Rightarrow\hspace{5mm}
  \phi_2\left(\left[\phi_3({\bf Y}) - {\bf 1}\cdot{\bf w}_3^T\right]\textsf{W}_3^{\dag}\right) &=&
      [{\bf 1},\f_1({\bf X}\bm{W}_1)]\cdot\bm{W}_2 .
  \label{eqn_omega_expr}
\end{eqnarray}
Consider again the case of separating the bias weights in the middle layer:
\begin{eqnarray}
  \phi_2\left(\left[\phi_3({\bf Y}) - {\bf 1}\cdot{\bf w}_3^T\right]\textsf{W}_3^{\dag}\right)
      - {\bf 1}\cdot{\bf w}_2^T &=&
      \f_1({\bf X}\bm{W}_1)\textsf{W}_2 \nonumber\\
      \Rightarrow\hspace{5mm}
  \left[ \phi_2\left(\left[\phi_3({\bf Y}) - {\bf 1}\cdot{\bf w}_3^T\right]\textsf{W}_3^{\dag}\right)
      - {\bf 1}\cdot{\bf w}_2^T \right]
      \textsf{W}_2^{\dag} &=&  \f_1({\bf X}\bm{W}_1) \nonumber\\
  \phi_1\left\{\left[ \phi_2\left(\left[\phi_3({\bf Y}) - {\bf 1}\cdot{\bf w}_3^T\right]\textsf{W}_3^{\dag}\right)
      - {\bf 1}\cdot{\bf w}_2^T \right]
      \textsf{W}_2^{\dag}\right\} &=&  {\bf X}\bm{W}_1 \nonumber\\
      \Rightarrow\hspace{5mm}
  {\bf X}^{\dag}\phi_1\left\{\left[ \phi_2\left(\left[\phi_3({\bf Y}) - {\bf 1}\cdot{\bf w}_3^T\right]\textsf{W}_3^{\dag}\right)
      - {\bf 1}\cdot{\bf w}_2^T \right]
      \textsf{W}_2^{\dag}\right\}  &=& {\bm{W}_1}.
  \label{eqn_var_omega}
\end{eqnarray}
By initializing $\bm{W}_2$ and $\bm{W}_3$, the hidden weights $\bm{W}_1$ can be estimated
using \eqref{eqn_var_omega}. After $\bm{W}_1$ has been estimated, $\bm{W}_2$ can be
computed based on \eqref{eqn_omega_expr} as follows:
\begin{eqnarray}
  \phi_2\left(\left[\phi_3({\bf Y}) - {\bf 1}\cdot{\bf w}_3^T\right]\textsf{W}_3^{\dag}\right) &=&
      [{\bf 1},\f_1({\bf X}{\bm{W}_1})]\cdot\bm{W}_2 \nonumber\\
  \left[{\bf 1},\f_1({\bf X}{\bm{W}_1})\right]^{\dag}\phi_2\left(\left[\phi_3({\bf Y}) - {\bf 1}\cdot{\bf w}_3^T\right]\textsf{W}_3^{\dag}\right)
      &=& {\bm{W}_2}
  \label{eqn_omega}
\end{eqnarray}
Once the weight matrices $\bm{W}_1$ and $\bm{W}_2$ are estimated, the output weights can
be determined as
\begin{equation}\label{eqn_output_weight_Nov11}
    {\bm{W}_3} = \left[{\bf 1},\f_2([{\bf 1},\f_1({\bf X}{\bm{W}_1})]
    \cdot{\bm{W}_2})\right]^{\dag}\phi_3({\bf
    Y}).
\end{equation}
The solution given by \eqref{eqn_output_weight_Nov11} boils down to least squares error
approximation because according to Lemma~\ref{lemma_LS}, multiplying the data matrix to
both sides of \eqref{eqn_3layer_net} is equivalent to projecting the original equation
onto the normal equation of sum of squared errors minimization \cite{StrangG1}. With this
observation on the three-layer network, we can generalize it to $n$-layer networks based
on the following assumption.

\newpage

\begin{assumption}\label{assump_n1}
  The functions $\f_i$, $i=1,\cdots,n$ operate elementwise on matrices so that
  $\f_n(\cdots\left[{\bf 1},\f_2([{\bf 1},\f_1({\bf X}\bm{W}_1)]\bm{W}_2)\right]\cdots\bm{W}_n)={\bf Y}$
  implies that for known $\bm{W}_{j}$, $j=1,\cdots,k-1$,
  there is a full rank matrix ${\bf A}_k$ such that ${\bf A}_k\bm{W}_k={\bf B}_k$
  for each $k=2,\cdots,n$ where\\
  \indent ${\bf A}_2=[{\bf 1},\f_1({\bf X}\bm{W}_1)]$, $\cdots$,
    ${\bf A}_n=[{\bf 1},f_{n-1}(\cdots [{\bf 1},f_2([{\bf 1},\f_1({\bf X}\bm{W}_1)]
    \bm{W}_2)]\cdots\bm{W}_{n-1})]$, \\
  \indent ${\scriptsize {\bf B}_2=
        \phi_1\left(
        \left[\phi_2(\cdots
        \left[\phi_{n-1}(
        \left[\phi_n({\bf Y}) - {\bf 1}\cdot{\bf w}_n^T\right]
        \textsf{\small W}_n^{\dag}) - {\bf 1}\cdot{\bf w}_{n-1}^T\right]
        \textsf{\small W}_{n-1}^{\dag}\cdots) - {\bf 1}\cdot{\bf w}_2^T\right]
        \textsf{\small W}_2^{\dag} \right)}$, \\
  \indent $\cdots$, ${\bf B}_n=\phi_n({\bf Y})$, \\
  and $\phi_k$, $k=2,\cdots,n$ are monotonic and injective functions
  that operate elementwise on their matrix domain.
\end{assumption}


\begin{thm} \label{thm_multiple}
Suppose $f_1\cdots f_n$ satisfy Assumptions~\ref{assump_11}--\ref{assump_n1}. Consider
the fully-connected $n$-layer network given by
\begin{equation}\label{eqn_Nlayer_net}
    {\bf G} = \f_n\left([{\bf 1},\f_{n-1}(\cdots\left[{\bf 1},\f_2([{\bf 1},\f_1({\bf
    X}\bm{W}_1)]\cdot\bm{W}_2)\right]\cdots\bm{W}_{n-1})]\cdot\bm{W}_n \right),
\end{equation}
where ${\bf X}\in\Real^{m\times (d+1)}$, $\bm{W}_1\in\Real^{(d+1)\times h_1}$,
$\bm{W}_2\in\Real^{(h_1+1)\times h_2}$, $\cdots$,
$\bm{W}_{n-1}\in\Real^{(h_{n-2}+1)\times h_{n-1}}$,
$\bm{W}_{n}\in\Real^{(h_{n-1}+1)\times q}$, ${\bf 1}=[1,...,1]^T\in\Real^{m\times 1}$,
and output ${\bf G}\in\Real^{m\times q}$ with target ${\bf Y}\in\Real^{m\times q}$. Then,
learning of the weights in \eqref{eqn_Nlayer_net} in terms of the least squares error
approximation boils down to solving the following set of cross-coupling equations
 {\footnotesize
\begin{equation} \label{eqn_soln_Nlayer_Theta}
\hspace{-3mm}\begin{array}{ccl}
  \bm{W}_1 &=& {\bf X}^{\dag}
        \phi_1\left(
        \left[\phi_2(\cdots
        \left[\phi_{n-1}(
        \left[\phi_n({\bf Y}) - {\bf 1}\cdot{\bf w}_n^T\right]
        \textsf{W}_n^{\dag}) - {\bf 1}\cdot{\bf w}_{n-1}^T\right]
        \textsf{W}_{n-1}^{\dag}\cdots) - {\bf 1}\cdot{\bf w}_2^T\right]
        \textsf{W}_2^{\dag} \right),   \\
  \bm{W}_2 &=& \left[{\bf 1},\f_1({\bf X}\bm{W}_1)\right]^{\dag}
        \phi_{2}\left(
        \left[\phi_{3}(\cdots
        \left[\phi_{n-1}(
        \left[\phi_n({\bf Y}) - {\bf 1}\cdot{\bf w}_n^T\right]
        \textsf{W}_n^{\dag}) - {\bf 1}\cdot{\bf w}_{n-1}^T\right]
        \textsf{W}_{n-1}^{\dag}\cdots \right.\right. \\
        && \left.\left. \hspace{106mm}
        ) - {\bf 1}\cdot{\bf w}_3^T\right] \textsf{W}_3^{\dag} \right),  \\
        &\vdots&  \\
   \bm{W}_{n-1}
        &=& \left[{\bf 1},f_{n-2}([\cdots f_2([{\bf 1},\f_1({\bf X}\bm{W}_1)]\bm{W}_2)]\cdots\bm{W}_{n-2})\right]^{\dag}
        \phi_{n-1}\left(\left[\phi_n({\bf Y}) - {\bf 1}\cdot{\bf w}_n^T\right]\textsf{W}_n^{\dag}\right),  \\
   \bm{W}_n &=& \left[\begin{array}{c} {\bf w}_n^T \\ \textsf{W}_n \\
        \end{array}\right]
        = \left[{\bf 1},f_{n-1}(\cdots [{\bf 1},f_2([{\bf 1},\f_1({\bf X}\bm{W}_1)]\bm{W}_2)]
        \cdots\bm{W}_{n-1})\right]^{\dag} \phi_n({\bf Y}).
\end{array}
\end{equation} }
\end{thm}
\begin{proof}
The proof follows from the derivation of the three-layer network by transforming or
inverting all the activation functions of every layer to express $\bm{W}_1$ in terms of
the sub-matrices $\textsf{W}_2$,...$\textsf{W}_n$ and vectors ${\bf w}_2$,...,${\bf
w}_n$. Then $\bm{W}_2$ is constructed based on $\bm{W}_1$,
$\textsf{W}_3$,...$\textsf{W}_n$ and ${\bf w}_3$,...,${\bf w}_n$ with activation one
layer near the inputs untransformed or un-inverted. The estimation of each increment of
the weight layer comes with one layer near the inputs untransformed or un-inverted. Such
construction follows through till $\bm{W}_n$. Hence the result.
\end{proof}

With this result, we propose a single pass algorithm for multilayer network learning
based on the Kernel-And-Range (KAR) space as follows.
\begin{center}
\begin{tabular}{lcl}
  \textbf{Algorithm} &:& {\texttt{KARnet}} \\
  \hline
  \textbf{Initialization} &:& Assign random weights to $\textsf{W}_2,\cdots,\textsf{W}_n$ and
        ${\bf w}_2$,...,${\bf w}_n$. \\
  \textbf{Learning} &:& Compute the network weights $\bm{W}_1$,...$\bm{W}_n$ sequentially using
    \eqref{eqn_soln_Nlayer_Theta}
    \\ && with $f=f_1=\cdots=f_n$ using `\texttt{logit}' activation and
    \\ && $\phi=\phi_1=\cdots=\phi_n$ using `\texttt{sigmoid}' transformation. \\
  \textbf{Net output} &:& Compute the network output ${\bf G}$ according to
    \eqref{eqn_Nlayer_net}.\\
  \hline \\
\end{tabular}
\end{center}

\noindent{\bf Remark 3: } Theorem~\ref{thm_multiple} shows that the FFN network learning
problem is a cross-coupling one, with weights in each layer depending on each and every
other layers. For the proposed learning, the learning outcome is dependent on the
initialization of the weight terms $\textsf{W}_2,\cdots,\textsf{W}_n$ and ${\bf
w}_2$,...,${\bf w}_n$. For simplicity in terms of implementation, we shall investigate
into a random initialization of these weight terms even though the algorithm does not
impose such a limitation. The representation of such a randomized network shall be
investigated in the sequel. \Box


\subsection{Network Representation} \label{sec_representation}

Here, we show the network representation capability before the experiments. In our
notations, the function $f=f_1=\cdots=f_n$ is operated elementwise on the matrix domain.


\begin{thm} (\textbf{Two-layer Network}) \label{thm_twolayer_representation}
Given $m$ distinct data samples of multivariate inputs with univariate output and suppose
$f_1({\bf X}\bm{W}_1)$ has full rank. Then there exists a fully-connected two-layer
network of linear output with at least $m$ adjustable weights that can represent these
data samples.
\end{thm}

\begin{proof}
Consider an augmented set of inputs, then the 2-layer network of linear output can be
written as:
\begin{equation}\label{eqn_2layer_linear_output}
    \bm{g} = f_1({\bf X}\bm{W}_1)\bm{w}_2,
\end{equation}
where ${\bf X}\in\Real^{m\times(d+1)}$ is the augmented network input matrix,
$\bm{W}_1\in\Real^{(d+1)\times h_1}$ contains the hidden weights,
$\bm{w}_2\in\Real^{h_1\times 1}$ contains the output weights, and
$\bm{g}\in\Real^{m\times 1}$ is the network output vector. Since $f_1({\bf X}\bm{W}_1)$
is assumed to be of full rank for the given data and weight matrix, it turns out that
only $m$ elements are required for $\bm{w}_2$ (i.e., $h_1=m$) for solving ${\bf y} =
f_1({\bf X}\bm{W}_1)\bm{w}_2$ uniquely given any target vector ${\bf y}$. Hence the
proof.
\end{proof}



\begin{thm} (\textbf{Multilayer Network}) \label{thm_fitting_multilayer}
Given $m$ distinct samples of input-output data pairs and suppose the adopted activation
function satisfies Assumption~\ref{assump_n1}. Then there exists a fully-connected
$n$-layer network of $q$ outputs with at least $m\times q$ adjustable weights that can
represent these data samples.
\end{thm}

\begin{proof}
By adopting a recursive notation of ${\bf A}_{k}=f_k([{\bf 1},{\bf A}_{k-1}]\bm{W}_k)$
for $k=2,3,\cdots,n-2,n-1$ with ${\bf A}_{1}=f_{1}({\bf X}\bm{W}_{1})$, the $n$-layer
network can be written as
\begin{equation}\label{eqn_2layer_outermost1}
    \bm{G} = f_n([{\bf 1},{\bf A}_{n-1}]\bm{W}_n)\ \ \ \in\Real^{m\times q}.
\end{equation}
Consequently, \eqref{eqn_2layer_outermost1} can be viewed as a system of linear equations
to learn the weight matrix $\bm{W}_n\in\Real^{h_{n-1}\times q}$ for fitting the target
${\bf Y}\in\Real^{m\times q}$ based on the transformation using $\phi_n$:
\begin{equation}\label{eqn_2layer_outermost2}
    \phi_n({\bf Y}) = [{\bf 1},{\bf A}_{n-1}]\bm{W}_n.
\end{equation}
To fully represent the $m$ number of data points labelled by $\phi_n({\bf Y})$, the
matrix $[{\bf 1},{\bf A}_{n-1}]$ needs to have at least rank $m$. This means that the
output weight matrix $\bm{W}_n$ requires at least $m$ (i.e., $h_{n-1}+1=m$) adjustable
weight elements for each of its $q$ columns corresponding to the output dimension. As for
the inner layer weights $\bm{W}_k$, $k=1,\cdots,n-1$, an identity matrix of $m\times m$
size would be suffice for transferring the input information from the innermost layer to
the $(n-1)^{th}$ layer. Since the compositional function $f_{n-1}(\cdots f_1({\bf
X}\bm{W}_1)\cdots)$ is of full rank according to Assumption~\ref{assump_n1}, $[{\bf
1},{\bf A}_{n-1}]$ in \eqref{eqn_2layer_outermost2} has full rank of size $m$. This means
that $\bm{W}_n$ in \eqref{eqn_2layer_outermost2} can be estimated uniquely.

In a nutshell, the $n$-layer network with $q$ outputs can be reduced to $q$ number of the
single-output network of two-layer in \eqref{eqn_2layer_linear_output} where each output
requires $m$ number of adjustable weights. Hence the proof.
\end{proof}


\subsection{Towards Understanding of Network Generalization}

In view of the good results observed in deep learning despite the huge number of weight
parameters, many attempts can be found in studying the generalization property of deep
networks (see e.g.,
\cite{Poggio10,Kawaguchi3,YBengio3,Neyshabur3,Foster1,Neyshabur2,Neyshabur1,YBengio2,LiHao1,Dziugaite1,Langford1}).
Based on our result on network representation, several properties related to network
generalization are observed.

\begin{enumerate}
\item The array structure of the fully-connected network in matrix form plays a key role
in constraining the complexity of model for approximation. Particularly, the number of
hidden nodes in each layer ($h_i$, $i=1,2,...,n$) determines the column size of the
weight matrix $\bm{W}_i$. This column size directly influences the effective parameter
size in each layer. Hence, the effective number of parameters (and the complexity of
approximation) does not grow exponentially (with the size of the gross network
parameters) with respect to the number of layers.

\item Thanks to the high redundancy of weights within each column, such an array
structure of the weight matrix offers numerous alternative solutions that attain the
desired fitting result. This accounts for the ease of fitting even with random
initialization of the weights (see the proposed algorithm and
\eqref{eqn_soln_Nlayer_Theta}) and leaves the remaining issue of generalization to the
representativeness of the training data. This property shall be exploited in the
algorithm for experimentation.

\item The representation capability of network is shown to be grounded upon the
well-known theory of linear algebra. However to our surprise, the nonlinearity in the
activation functions actually contributes to the well-posedness (such as making the
matrix invertible) of the problem than otherwise.

\item Capitalized on the minimum norm property according to Lemma~\ref{lemma_LS} and
Lemma~\ref{lemma_LS_matrix}, no explicit regularization has been incorporated into the
learning. The least squares approximation (section~\ref{sec_LS_approx}) in the primal
space (with covariance matrix ${\bf A}^T{\bf A}\in\Real^{d\times d}$ in
\eqref{eqn_d_space_primal}) and the dual space (with Gram matrix ${\bf A}{\bf
A}^T\in\Real^{m\times m}$ in \eqref{eqn_m_space_dual}) has, indeed, long been known in
the theory of linear systems \cite{Albert1}. Unfortunately, the effectiveness of such
utilization of the primal-dual pair has been masked by the popularity of weight decay
regularization. The observation in \cite{ZhangCY1} which compares between the regularized
results with those of non-regularized ones is not surprising due to the minimum norm
solution nature (section~\ref{sec_LS_approx}) which has long been known in the
statistical community.

\end{enumerate}

Given the mapping capability of the network together with the above observations, we
conjecture that network generalization is hinged upon the \emph{model complexity} and the
similarity/difference between the \emph{distributions} of the training and testing data
sets (such as \emph{optimism} defined by \cite{Efron3}).

\section{Case Study} \label{sec_synthetic}

In this study, we show case the results of the proposed learning algorithm
\texttt{KARnet} on two benchmark examples. The goal of the first example is to observe
the decision surface and the training CPU time for the compared methods on mapping
linearly non-separable data. The chosen data set is the well-known XOR problem. The
second example aims to validate the network representation theory. The chosen problem is
the iris flower data set from the UCI machine learning repository \cite{UCI1b}. All the
experimental studies were carried out on an Intel Core i7-6500U CPU at 2.59GHz with 8G of
RAM.

\subsection{The XOR problem} \label{sec_XOR}

In this example, the proposed learning algorithm \texttt{KARnet} and the conventional
\texttt{\texttt{feedforwardnet}} from the Matlab toolbox \cite{Matlab} are evaluated in
terms of their decision landscapes. Both the compared methods use the `\texttt{logit}'
(i.e., $f=\log(x/(1-x))$ as the activation function where its functional inverse is
$f^{-1}=1/(1+e^{-x})=$ `\texttt{sigmoid}') for $x\in[0,1)$. The training of
\texttt{feedforwardnet} uses the default Levenberg-Marquardt `\texttt{trainlm}' search
and stopping. For both learning algorithms, a two-layer network with one hidden layer
that consists of two nodes ($h_1=2$) and a five-layer network with each layer composing
of five hidden nodes ($h_1=h_2=h_3=h_4=3$, $q=1$, we call it a 3-3-3-3-1 structure) are
evaluated.

The learned decision surfaces for the two-layer \texttt{KARnet} and
\texttt{feedforwardnet} are shown in Fig.~\ref{fig_XOR_2n10layer}(a) and (b)
respectively. Due to the perfect symmetry of the XOR data points which can cause
degenerative pseudoinverse computation, a small perturbation has been included to the
data points where the training data points are $(x_1,x_2)\in$ $\{(0,0)$,
$(0.9991,0.9991)$, $(0.9990,0)$, $(0,0.9990)\}$ for respective target values $y\in\{0, 0,
1, 1\}$. The decision surface in Fig.~\ref{fig_XOR_2n10layer}(a) shows a perfect fitting
of all the data points in the two classes for \texttt{KARnet}. Contrastingly, the
decision surface of \texttt{feedforwardnet} in Fig.~\ref{fig_XOR_2n10layer}(b) shows a
complex fitting of the data points. This could be due to an early stopping of search at
steep error surface of the \texttt{logit} function. The trained outputs of the two-layer
networks for the four data points are [$0$, $0$, 1, 1] and [$-0.0758$, $0.0665$,
$0.4302$, $1.0988$] respectively for \texttt{KARnet} and \texttt{feedforwardnet}.

The learned output surfaces for the two five-layer networks namely, \texttt{KARnet} and
\texttt{feedforwardnet} are shown in Fig.~\ref{fig_XOR_2n10layer}(c) and (d)
respectively. The trained outputs for the five-layer \texttt{KARnet} and
\texttt{feedforwardnet} are respectively [$0$, 0, 1, 1] and [$-0.1527$, 0.0331,
$-0.2331$, 0.0296]. This result, again, shows inadequate stopping of the gradient-based
search.

\begin{figure}[hhh]
  \begin{center}
\begin{tabular}{cc}
  \epsfxsize=6.08cm
  \hspace{-5mm}
  \epsffile[29     4   381   292]{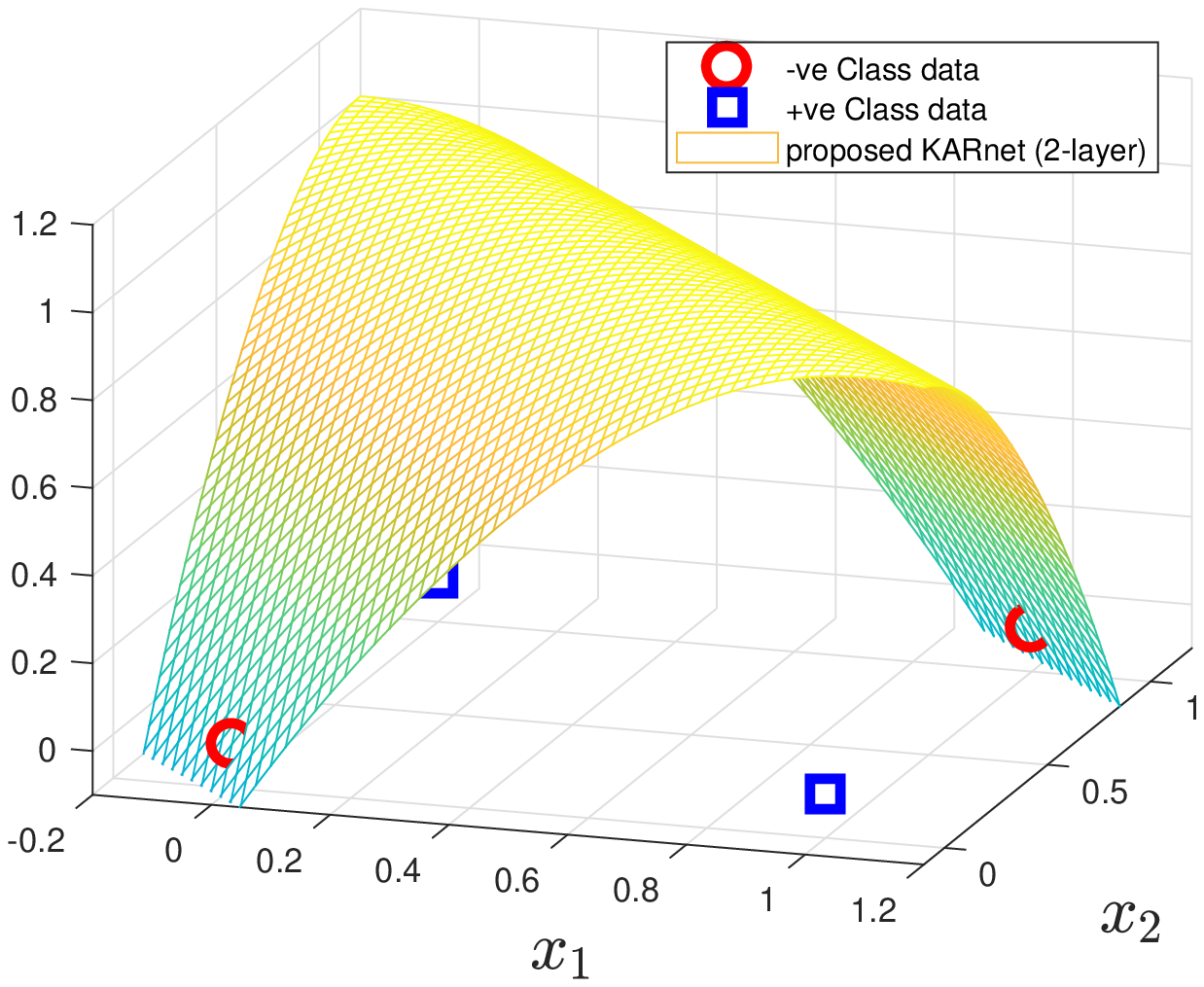} &
  \epsfxsize=6.08cm
  \hspace{0cm}
  \epsffile[29     4   381   292]{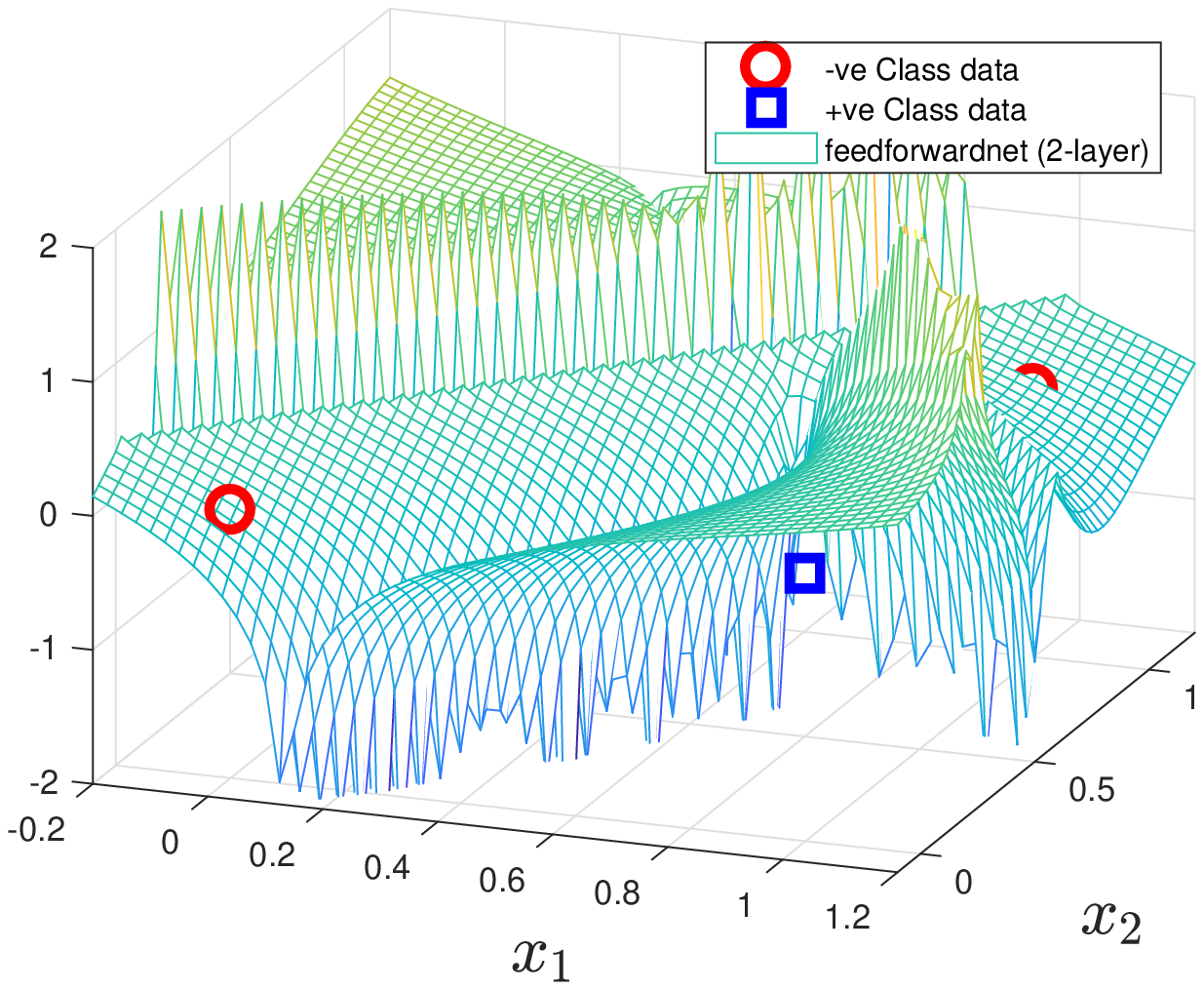}
  \\*[0mm] \hspace{3mm} (a) & \hspace{13mm} (b) \\*[5mm]
  \epsfxsize=6.08cm
  \hspace{-5mm}
  \epsffile[27     5   381   292]{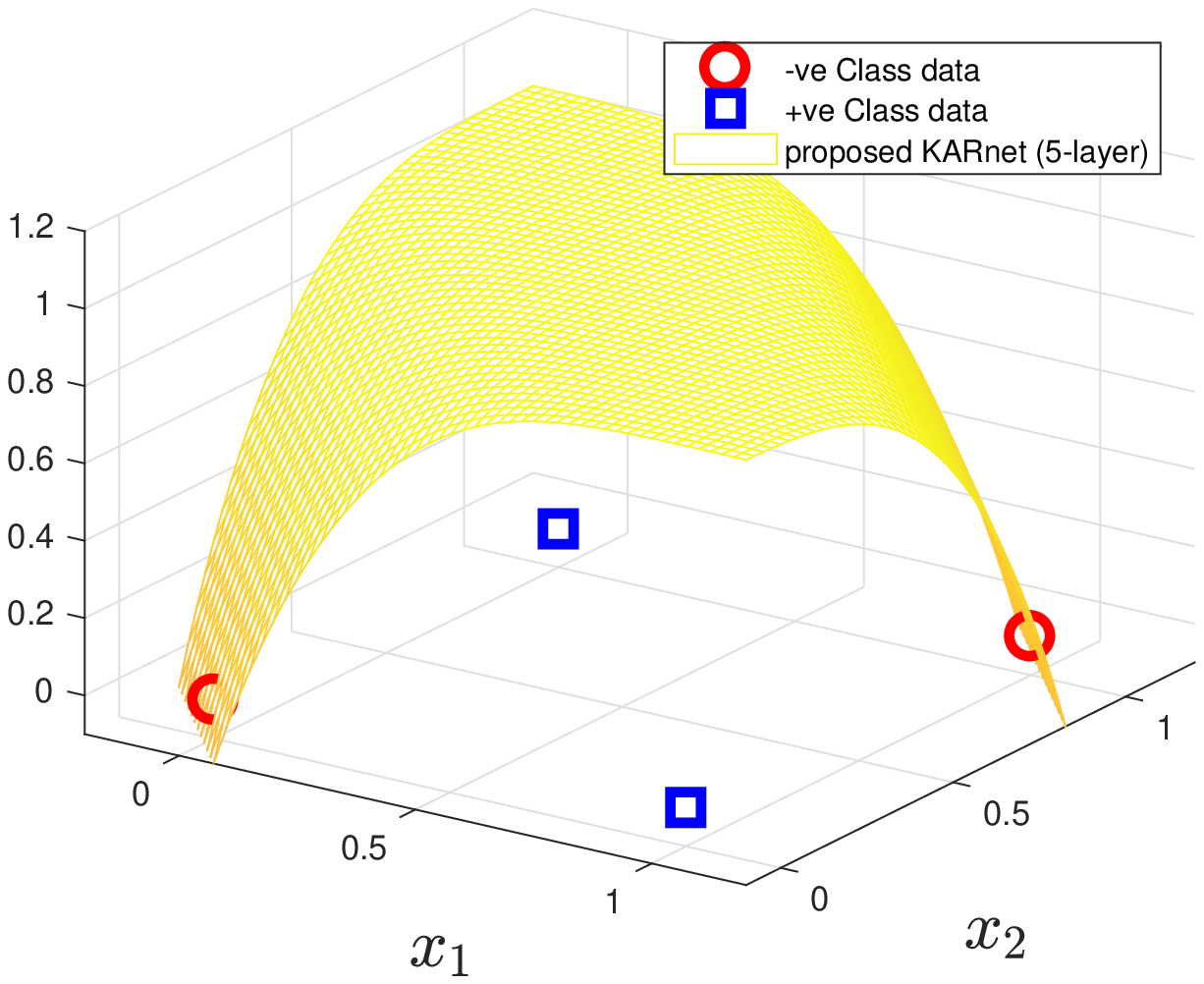} &
  \epsfxsize=6.08cm
  \hspace{0cm}
  \epsffile[27     5   381   292]{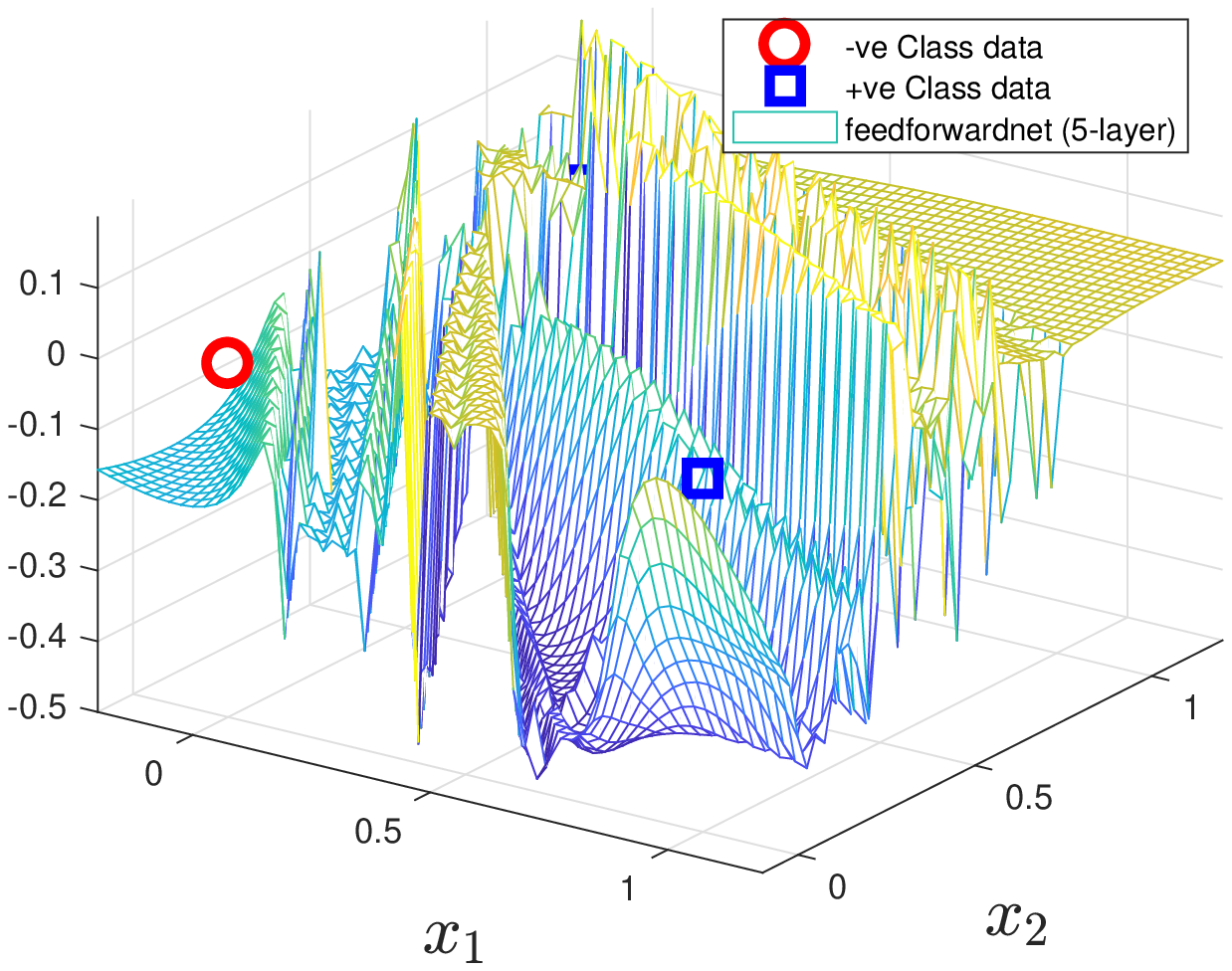}
  \\*[0mm] \hspace{3mm} (c) & \hspace{13mm} (d)
\end{tabular}
  \caption{Decision surface plots of two-layer/five-layer feedforward networks
  (proposed \texttt{KARnet}($\f=logit$, $\g=sigmoid$) and \texttt{feedforwardnet})
  trained by the proposed method and trainlm.}
  \label{fig_XOR_2n10layer}
  \end{center}
\end{figure}

For the two-layer networks, the clocked training CPU times are 0.0625 and 1.0781 seconds
for \texttt{KARnet} and \texttt{feedforwardnet} respectively. For the five-layer
networks, the training CPU times for \texttt{KARnet} and \texttt{feedforwardnet} are
0.1250 and 1.0938 seconds respectively. Certainly, these computational time results show
the efficiency of the proposed training comparing with the gradient descent search even
without optimizing our code implementation.

\subsection{Network Representation: Iris Plant Data} \label{sec_iris}

The iris plant data is a well-known benchmark data set in classification with small
sample size ($150$ in total) and small dimension ($d=4$). Among the 150 samples, each of
the three categories ($q=3$) contains 50 samples. For illustration purpose, 90 samples
(within which 30 samples per category) are used for training and the remaining 60 are
used as test samples. A two-layer network is used in this study. In order to illustrate
the systems of over-determined and under-determined cases, the number of nodes in the
hidden layer is chosen as $h_1\in\{79,80,\cdots,93\}$. This gives a hidden weight
$\bm{W}_1$ of size $(4+1)\times h_1$. Since the training sample size is $m=90$, the case
for which $h_1=90$ is an equal-determined case in view of the output weights, whereas
those with lower and higher sizes are respectively the over-determined and
under-determined cases.

The two-layer network is learned using a randomized $\bm{W}_1$ with an one-versus-all
indicator target. Fig.~\ref{fig_over_under_deterimined} shows two cases of learning
outputs for $h_1=89$ (Fig.~\ref{fig_over_under_deterimined}(a)) and $h_1=90$
(Fig.~\ref{fig_over_under_deterimined}(b)). Here, we observe that the outputs have a
perfect fit to all data samples when $h_1=90$. Conversely, we do not see a perfect fit
when $h_1=89$ because there are more data than the \emph{number of effective
parameters}.\footnote{We use the term \emph{number of effective parameters} for the
purpose of differentiating it with the actual number of weight elements in total.
Essentially, the \emph{effective parameters} bear the spirit of the required matrix
dimension for identifiability.}

\begin{figure}[hhh]
  \begin{center}
\begin{tabular}{cc}
  \epsfxsize=6.08cm
  \hspace{-5mm}
  \epsffile[22     2   385   312]{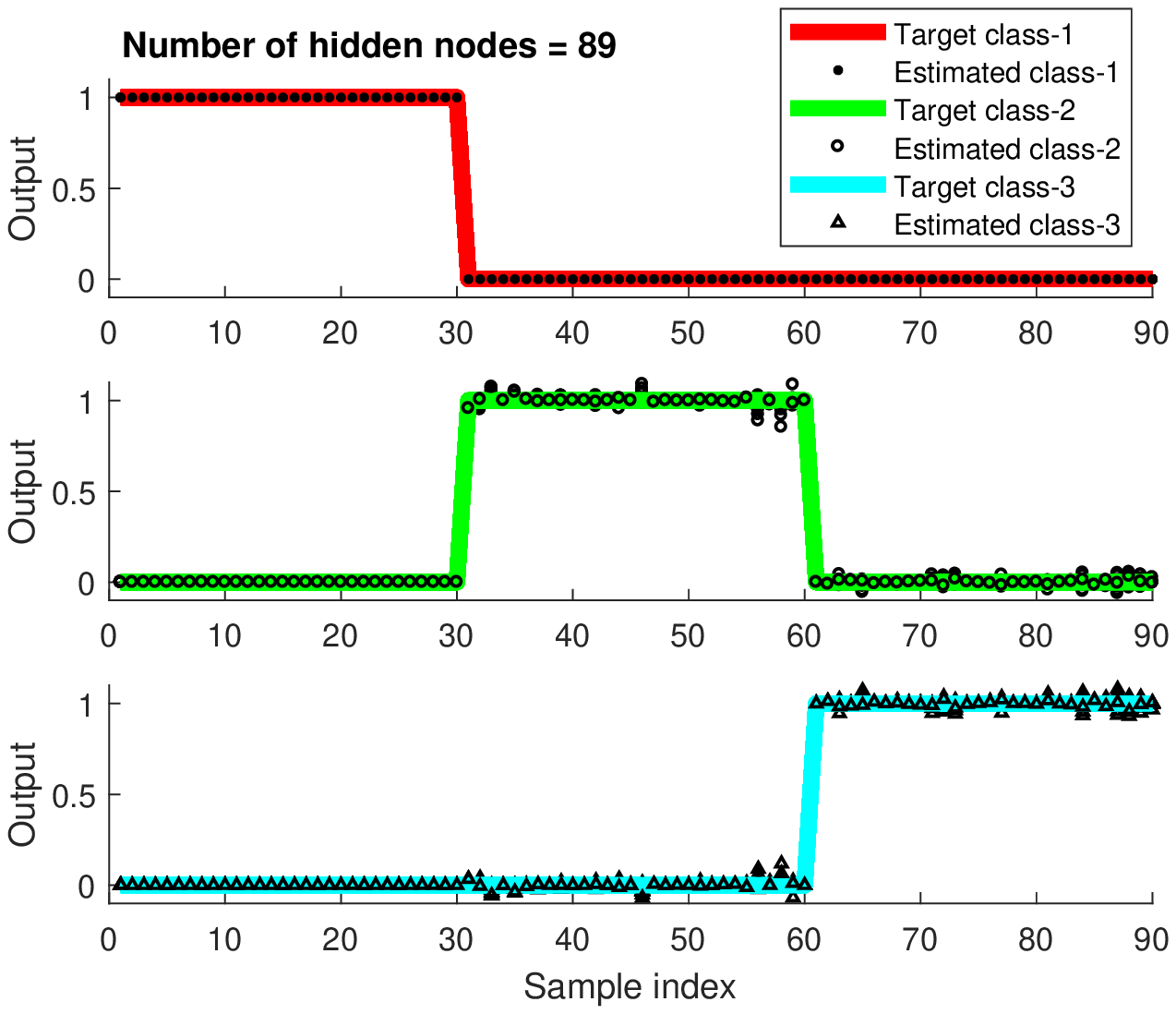} &
  \epsfxsize=6.08cm
  \hspace{0cm}
  \epsffile[22     2   385   312]{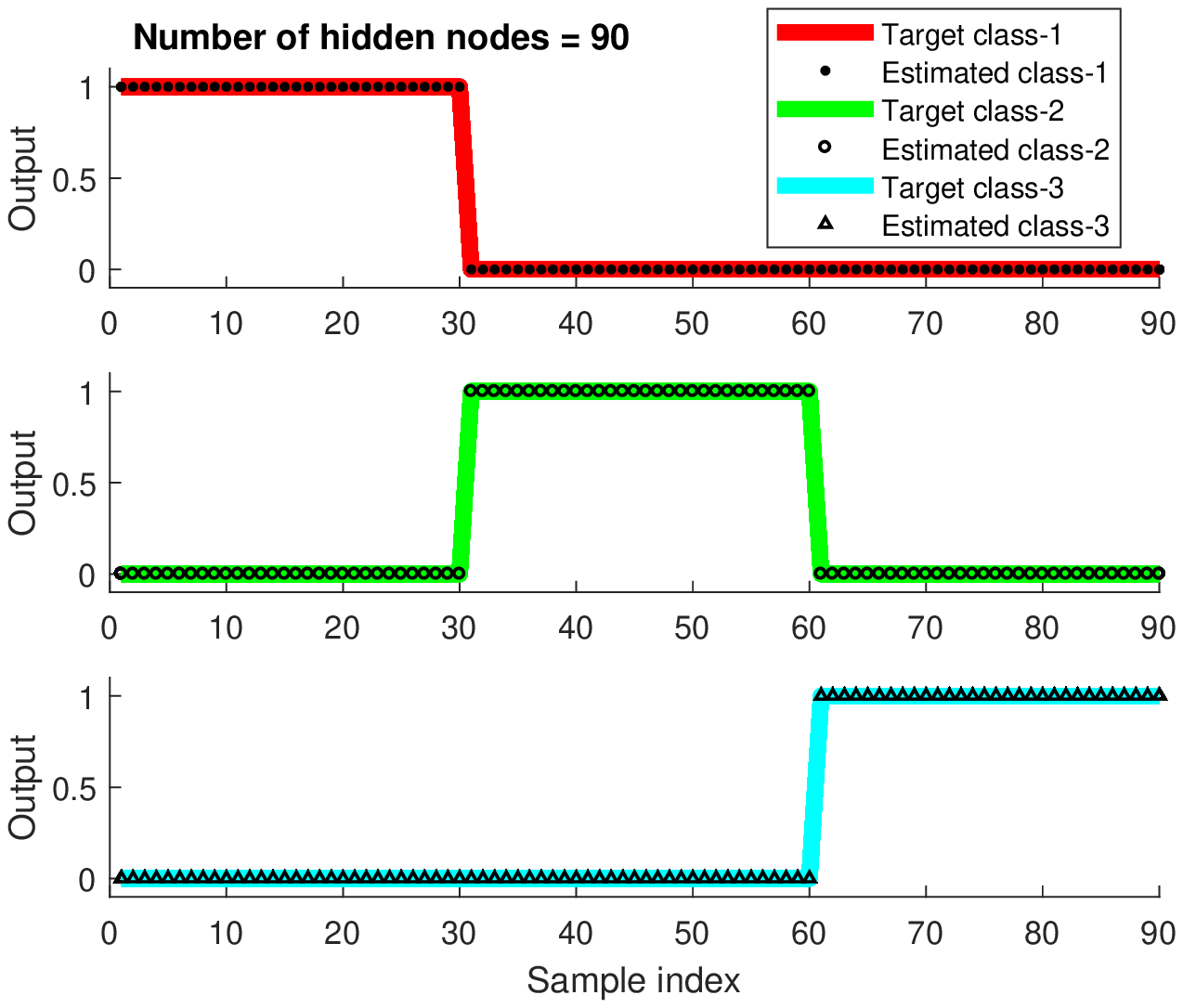}
  \\*[0mm] \hspace{3mm} (a) & \hspace{13mm} (b)
\end{tabular}
  \caption{Target outputs and learned outputs of a 2-layer \texttt{KARnet}.}
  \label{fig_over_under_deterimined}
  \end{center}
\end{figure}

Fig.~\ref{fig_MSE2} shows 10 Monte Carlo trials of the quantitative fitting results in
terms of the Sum of Squared Errors (SSE) and the classification error count rates for
both training and testing. Based on the theory of representation, the two-layer network
shall fit the target value of all the data samples when $h_1\geq m$ ($m=90$ here) for
each of the 3 output columns. This is clearly observed from the SSE plot where SSE=0 is
observed for $h_1\geq 90$. Since this is a classification problem, the training and test
classification error rates are also included in the plot. These results show that the
training has attained a zero classification error count while the test classification
error rate is not significantly affected by the goodness of fit in terms of SSE.

\begin{figure}[hhh]
  \begin{center}
  \epsfxsize=8.8cm
  \epsffile[22     4   385   295]{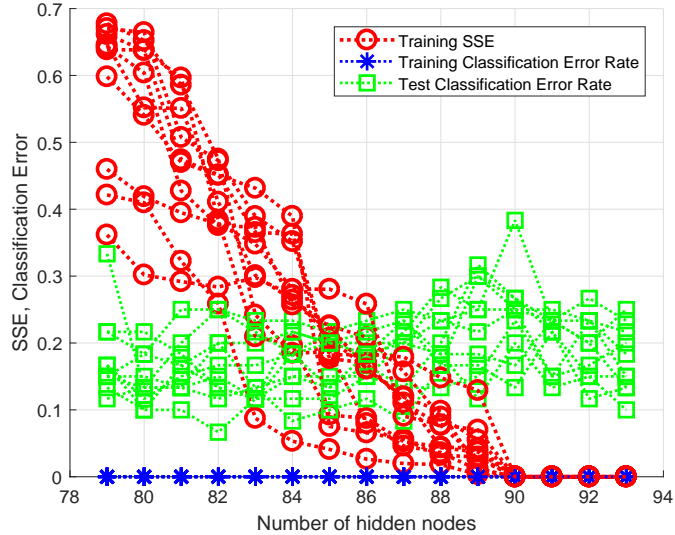}
  \caption{The SSE and the classification error count rates for a two-layer network}
  \label{fig_MSE2}
  \end{center}
\end{figure}

Next, we train a five-layer network with the first four layers having randomized weights.
Except for $\bm{W}_5$, each of the weights $\bm{W}_1$, ..., $\bm{W}_4$ has a similar
column size to that of the first layer. Fig.~\ref{fig_MSE5} shows the SSE and
classification error rates. This result shows that zero SSE can be attained at column
size ($h_i$, $i=1,...,4$) lower than 90! This suggests the number of hidden weight nodes
(columns) may be traded by the layer depth. This opens up an important topic for future
study.

\begin{figure}[hhh]
  \begin{center}
  \epsfxsize=8.8cm
  \epsffile[22     4   385   295]{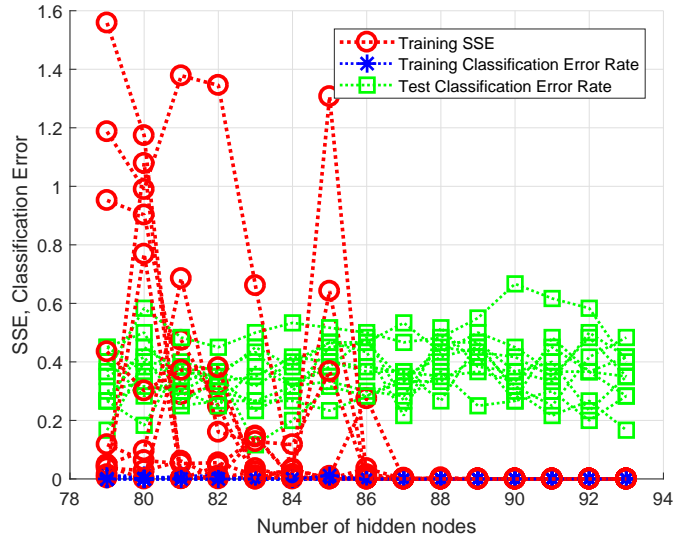}
  \caption{The SSE and the classification error count rates for a five-layer network}
  \label{fig_MSE5}
  \end{center}
\end{figure}

\section{Experiments} \label{sec_expts}

In this section, the proposed learning algorithm \texttt{KARnet} is evaluated on
real-world data sets taken from the public domain. Due to the limited computing facility
(lack of GPU), only data sets ranging from small to medium sample sizes are considered.
The goals of this experimentation include (i) to verify the practical feasibility of the
formulation on real-world data of relatively small dimension ($d\leq 64$); (ii) if
practically feasible, to observe how well it compares with the state-of-the-art
backpropagation algorithm that utilizes the same `\texttt{logit}' activation function
with that in \texttt{KARnet}. (iii) to record the computational training CPU times for
the proposed algorithm and the gradient based network learning counterpart. Since the
fine tuning of network and learning generalization are topics for further research, the
scope of current empirical investigation is limited to observing the practical
feasibility, the basic behavior in terms of raw attainable accuracy without fine tuning,
and the computational efficiency of the proposed algorithm. A total of 42 data sets from
the UCI machine learning repository \cite{UCI1b} are adopted for this study.

\subsection{UCI Data Sets}

The adopted data sets from the UCI machine learning repository \cite{UCI1b} is summarized
in Table~\ref{table_summary_attributes} with their pattern attributes \cite{Toh39}. The
experimental goals (i) and (ii) are evaluated by observing the prediction accuracy of the
algorithm.

\begin{table}
\begin{center}\vspace{-6mm}
\caption{Summary of UCI \cite{UCI1b} data sets and chosen hidden layer sizes based on
cross-validation within the training set} \label{table_summary_attributes}{\tiny
\begin{tabular}{|cl|cccc|cc|c|c|} \hline
    &                     & (i)       & (ii)      & (iii)     & (iv)    &
    \multicolumn{2}{c}{2-layer} \vline & 3-layer    & 4-layer \\
    &                     &           &           &           &         &
    (hidd.size)           & (hidd.size) & (hidd.size) & (hidd.size) \\
    & Database name       & {\#}cases & {\#}feat & {\#}class& {\#}miss & FFnet & \texttt{KARnet} & \texttt{KARnet} & \texttt{KARnet} \\ \hline
1.  & Shuttle-l-control   & 279(15)   & 6        & 2        & no       &  30  &  1 &   2 &  10  \\
2.  & BUPA-liver-disorder & 345       & 6        & 2        & no       &  80  &  2 &   1 &   2  \\
3.  & Monks-1             & 124(432)  & 6        & 2        & no       &  20  &  3 &   1 &   1  \\
4.  & Monks-2             & 169(432)  & 6        & 2        & no       &  30  &  1 & 500 & 500  \\
5.  & Monks-3             & 122(432)  & 6        & 2        & no       &  30  & 80 &   1 &   3 \\
6.  & Pima-diabetes       & 768       & 8        & 2        & no       &  80  &  2 &   1 &  10 \\
7.  & Tic-tac-toe         & 958       & 9        & 2        & no       & 200  & 20 &  20 &  20 \\
8.  & Breast-cancer-Wiscn & 683(699)  & 9(10)    & 2        & 16       &  50  &  1 &   1 &  20 \\
9.  & StatLog-heart       & 270       & 13       & 2        & no       &  80  &  1 &  10 &   5 \\
10. & Credit-app          & 653(690)  & 15       & 2        & 37       &  80  &  3 &   1 &  10 \\
11. & Votes               & 435       & 16       & 2        & yes      &  10  &  1 &  20 &   1 \\
12. & Mushroom            & 5644(8124)& 22       & 2        & attr{\#}11 & 200&  1 &  30 &   5 \\
13. & Wdbc                & 569       & 30       & 2        & no       &  30  &  1 &   1 &   3 \\
14. & Wpbc                & 194(198)  & 33       & 2        & 4        &  80  &  2 &  20 & 500 \\
15. & Ionosphere          & 351       & 34       & 2        & no       &  80  &  5 &   5 &  10 \\
16. & Sonar               & 208       & 60       & 2        & no       &  80  &  1 &  50 &   2 \\
\hline
17. & Iris                & 150       & 4        & 3        & no       & 500  &  5 &   5 &   2 \\
18. & Balance-scale       & 625       & 4        & 3        & no       & 500  & 20 &  100&   3 \\
19. & Teaching-assistant  & 151       & 5        & 3        & no       & 500  &  2 &  80 & 200 \\
20. & New-thyroid         & 215       & 5        & 3        & no       & 500  & 30 &   5 &   2 \\
21. & Abalone             & 4177      & 8        & 3(29)    & no       & 80   & 10 &  20 &   3 \\
22. & Contraceptive-methd & 1473      & 9        & 3        & no       &   5  & 80 &  10 &   5 \\
23. & Boston-housing      & 506       & 12(13)   & 3(cont)  & no       & 500  & 10 &  10 &  20 \\
24. & Wine                & 178       & 13       & 3        & no       & 500  &  5 &   5 &   5 \\
25. & Attitude-smoking$^+$& 2855      & 13       & 3        & no       &   2  &  1 &  50 &  10 \\
26. & Waveform$^+$        & 3600      & 21       & 3        & no       &  50  & 80 &  10 &   5 \\
27. & Thyroid$^+$         & 7200      & 21       & 3        & no       & 100  & 10 &  10 &  20 \\
28. & StatLog-DNA$^+$     & 3186      & 60       & 3        & no       & 200  & 80 &   5 &  50 \\
\hline
29. & Car                 & 2782      & 6        & 4        & no       & 500  & 50 & 500 & 500 \\
30. & StatLog-vehicle     & 846       & 18       & 4        & no       & 500  & 100&  50 &   5 \\
31. & Soybean-small       & 47        & 35       & 4        & no       &  50  &  1 &   1 &   1 \\
32. & Nursery             & 12960     & 8        & 4(5)     & no       & 500  & 30 &  30 &  10 \\
33. & StatLog-satimage$^+$& 6435      & 36       & 6        & no       &  80  & 100&  50 &  20 \\
34. & Glass               & 214       & 9(10)    & 6        & no       & 100  & 30 &   5 &  10 \\
35. & Zoo                 & 101       & 17(18)   & 7        & no       & 100  & 500&  10 &  80 \\
36. & StatLog-image-seg   & 2310      & 19       & 7        & no       & 100  & 100&  50 &  20 \\
37. & Ecoli               & 336       & 7        & 8        & no       & 500  & 30 &  10 &   5 \\
38. & LED-display$^+$     & 6000      & 7        & 10       & no       & 500  & 30 &  20 &  10 \\
39. & Yeast               & 1484      & 8(9)     & 10       & no       & 500  & 500&  30 &  30 \\
40. & Pendigit            & 10992     & 16       & 10       & no       & --   & 500&  80 & 200 \\
41. & Optdigit            & 5620      & 64       & 10       & no       & --   & 500& 100 &  30 \\
42. & Letter              & 20000     & 16       & 26       & no       & --   & 500& 200 & 200 \\
\hline
\end{tabular} } {\tiny
\begin{tabular}{llp{12.5cm}}
(i-iv) &: & (i) Total number of instances, i.e. examples, data points, observations
(given number of instances). Note: the number of instances used is larger than the given
number of instances when we expand those ``don't care'' kind of attributes in some data
sets; (ii) Number of features used, i.e. dimensions, attributes (total number of features
given); (iii) Number of classes (assuming a discrete class variable);
(iv)  Missing features; \\
$+$ &:   & Accuracy measured from the given training and test set instead of 10-fold
validation (for large data cases with test set containing at least 1,000 samples); \\
FFnet &: & The \texttt{feedforwardnet} from the Matlab's toolbox; \\
hidd.size &: & For 3layer and 10layer networks, the number of hidden layer nodes are
chosen similarly for simplicity of study; \\
Note &:  & Data from the Attitudes Towards Smoking Legislation Survey - Metropolitan
Toronto 1988, which was funded by NHRDP (Health and Welfare Canada), were collected by
the Institute for Social Research at York University for Dr. Linda Pederson and Dr.
Shelley Bull.
\end{tabular} }
\end{center}
\end{table}

\subsection*{(i) The prediction performance}

In this experiment, the prediction performance of the proposed \texttt{KARnet} is
compared with that of the \texttt{feedforwardnet} from the Matlab's toolbox \cite{Matlab}
using a two-layer fully-connected network structure. Both the networks adopted the
\emph{logit} activation function in this study. The performance is evaluated based on 10
trials of 10-fold stratified cross-validation tests for each of the data set. The
selection of the number of hidden nodes is based on another 10-fold cross-validation
within the training set.

The chosen hidden node size (based on a search within the set $\{$1, 2, 3, 5, 10, 20, 30,
50, 80, 100, 200, 500$\}$ by cross-validation using only the training set) is shown in
the third column-group of Table~\ref{table_summary_attributes}. For both the two-layer
networks, the selected size of hidden nodes appear to be much different for
\texttt{KARnet} and \texttt{Feedforwardnet} among the data sets. Apparently, there seems
to have no correlation between the choice of hidden node sizes for the two networks. The
choice of the size of the hidden nodes is dependent on the training accuracy where some
data sets appear to be more over-fit than others.

The average results of prediction accuracies (based on 10 trials of 10-fold
cross-validation) for such choice of hidden node size for \texttt{KARnet} and
\texttt{Feedforwardnet} are shown in Fig.~\ref{fig_compare1}. For the
\texttt{Feedforwardnet}, the accuracies of data sets 40, 41 and 42 are not available due
to the ``out of memory'' limitation for the chosen network size of 500 hidden nodes on
the current computing platform. From these results, the accuracies of \texttt{KARnet}
appear to be significantly better than that of \texttt{Feedforwardnet} for most of the
data sets. The gross average accuracies for \texttt{KARnet} and \texttt{Feedforwardnet}
are respectively 79.39\% and 66.99\% based on the first 39 data sets. These results
suggest relatively good generalization performance of \texttt{KARnet} than that of
\texttt{Feedforwardnet} based on the chosen activation function.

\begin{figure}[hhh]
  \begin{center}
  \epsfxsize=12cm
  \epsfysize=6cm
  \epsffile[58     5   626   304]{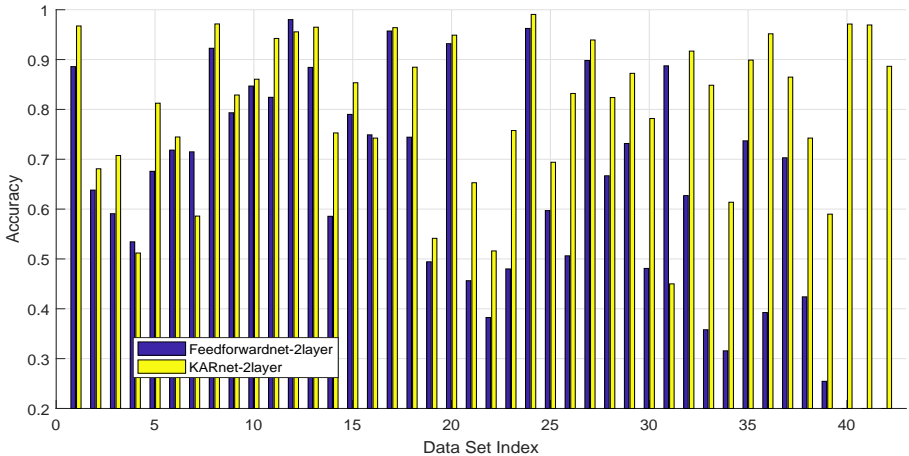}
  \caption{Comparing \texttt{KARnet} with \texttt{Feedforwardnet} at their best chosen hidden nodes size
  based on cross-validation using only the training data. The accuracies of data sets 40, 41 and 42 for
  \texttt{Feedforwardnet} are not shown due to the ``out of memory'' limitation for the current
  computing platform.}
  \label{fig_compare1}
  \end{center}
\end{figure}

In the next experiment, we observe the generalization performance of \texttt{KARnet} when
it uses the same hidden node size as that of \texttt{Feedforwardnet}.
Fig.~\ref{fig_compare2} shows the average accuracies of \texttt{KARnet} which are plotted
along with those of \texttt{Feedforwardnet}. It is clear from the figure that
\texttt{KARnet} outperforms \texttt{Feedforwardnet} for most data sets. The gross average
accuracy for \texttt{KARnet} based on 39 data sets is 80.03\% (comparing with 66.99\% for
\texttt{Feedforwardnet}).

\begin{figure}[hhh]
  \begin{center}
  \epsfxsize=12cm
  \epsfysize=6cm
  \epsffile[57     5   618   303]{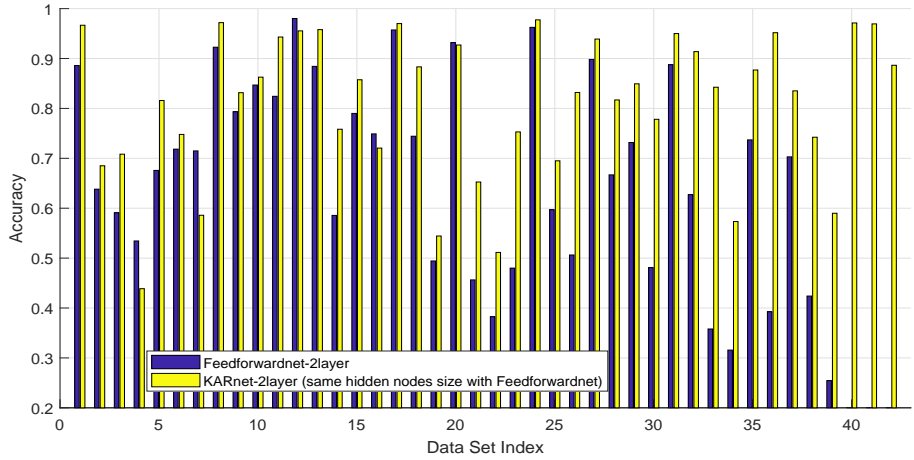}
  \caption{Comparing the accuracy of \texttt{KARnet} with that of \texttt{Feedforwardnet} at
  the same hidden nodes size.}
  \label{fig_compare2}
  \end{center}
\end{figure}

\subsection*{(ii) The training CPU time}

After observing the feasibility of network learning on the UCI data sets, we move on to
observe the training CPU time clocked for both the studied methods.
Fig.~\ref{fig_compare_CPU1} shows the training CPU times for \texttt{KARnet} and
\texttt{Feedforwardnet} at their optimized hidden node size based on cross-validation
using the training set. Fig.~\ref{fig_compare_CPU2} shows the training CPU times at the
same hidden node size for the two studied methods. For the optimized network size, the
gross average CPU times over the first 39 data sets are respectively 0.0211 seconds and
180.37 seconds for \texttt{KARnet} and \texttt{Feedforwardnet}. When the \texttt{KARnet}
uses the same network size with that of \texttt{Feedforwardnet}, the gross average CPU
times over the first 39 data sets is 0.8003 seconds. This shows a 225 times faster in
training speed for \texttt{KARnet} over \texttt{Feedforwardnet}. The gain from the
training speed is attributed to the closed-form and non-iterative learning.

\begin{figure}[hhh]
  \begin{center}
  \epsfxsize=12cm
  \epsfysize=6cm
  \epsffile[53     5   630   306]{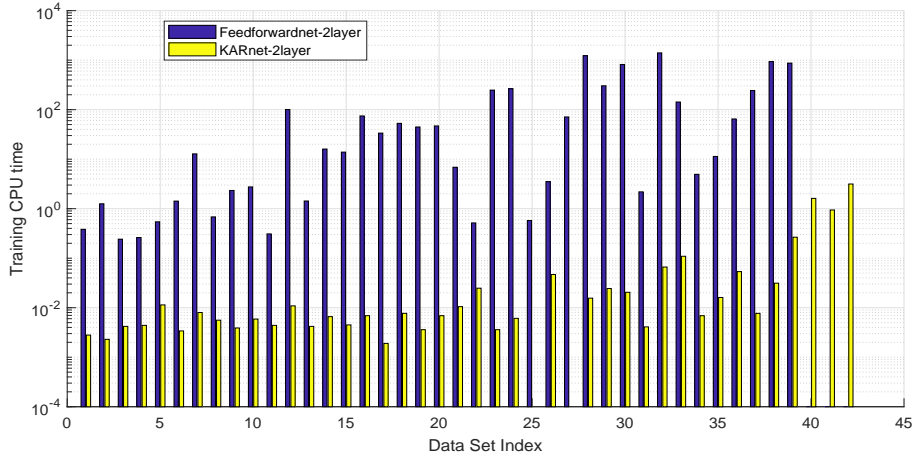}
  \caption{Comparing the training CPU time (seconds) of \texttt{KARnet} with that of \texttt{Feedforwardnet} at their
  cross-validation optimized size of hidden nodes.}
  \label{fig_compare_CPU1}
  \end{center}
\end{figure}

\begin{figure}[hhh]
  \begin{center}
  \epsfxsize=12cm
  \epsfysize=6cm
  \epsffile[52     5   626   305]{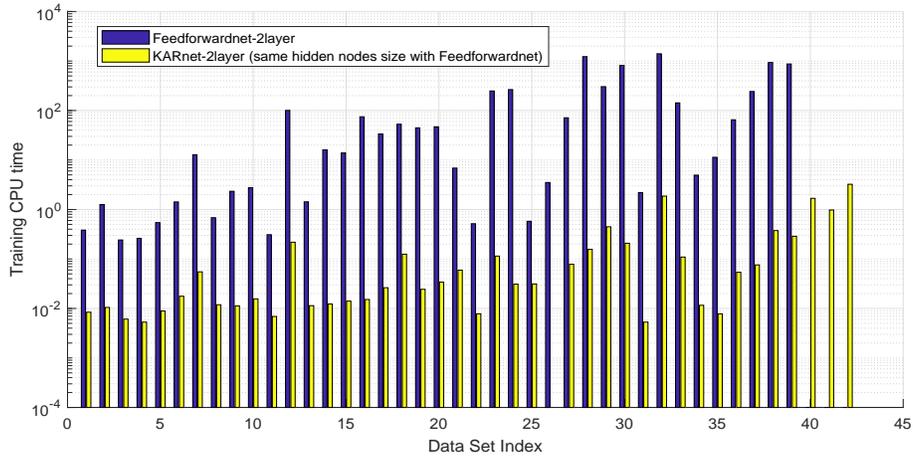}
  \caption{Comparing the training CPU time (seconds) of \texttt{KARnet} with that of \texttt{Feedforwardnet}
  when both the networks have the same size of hidden nodes.}
  \label{fig_compare_CPU2}
  \end{center}
\end{figure}

\subsection*{(iii) The effect of deep layers}

Having verified the feasibility and the prediction performance with the training
processing time, we move a step further to observe the effect of a deeper structure
utilizing an exponentially increasing number of hidden nodes towards the input layer.
Specifically, the two-layer network uses a structure of $[h,\ q]$ for the hidden layer
and the output layer respectively. The three-layer network uses a $[2\times h,\ h,\ q]$
structure and the four-layer network uses a $[4\times h,\ 2\times h,\ h,\ q]$ structure.
The dimension of output is $q$, and $h$ is determined based on a cross-validation search
within the training set for $h\in\{$1, 2, 3, 5, 10, 20, 30, 50, 80, 100, 200, 500$\}$.
The prediction accuracies for \texttt{KARnet} with two-, three- and four-layers are shown
in Fig.~\ref{fig_compare_deeplayers}. Except for a few data sets, the results show
insignificant variation of prediction accuracy according to the change in network depth.

\begin{figure}[hhh]
  \begin{center}
  \epsfxsize=12cm
  \epsfysize=6cm
  \epsffile[58     5   622   304]{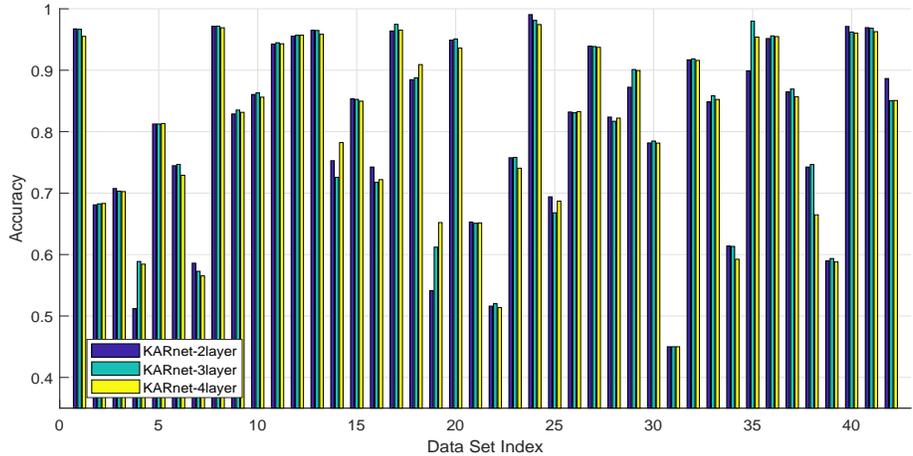}
  \caption{The effect of layers on accuracy performance.}
  \label{fig_compare_deeplayers}
  \end{center}
\end{figure}

\subsection{Summary of Results and Discussion}

With the experimental feasibility and learning efficiency validated, we shall summarize
the results obtained, from both theoretical and experimental perspectives. The
theoretical contributions of this work can be summarized as follows:
\begin{itemize}
    \item Kernel-And-Range space estimation: A theoretical framework has been established to relate the solution given by the
    kernel-and-range space estimation to that based on the least squares error estimation. The importance
    of such an establishment is that it enables a gradient-free network learning process where the
    formally inevitable iterative search process can be by-passed.

    \item Understanding of FFN learning: The proposed kernel-and-range space estimation has been applied to multilayer
    network learning wherein an analytical formulation is obtained. Based on this formulation,
    we found that solving the network training problem boils down to solving a set
    cross-coupling system of linear equations for the feedforward structure.

    \item An analytic learning algorithm: Utilizing the above results, a single-pass algorithm has been proposed to train
    the multilayer feedforward network. Essentially, the algorithm capitalizes on the random
    initialization to decouple the process of network weight estimation. This has been grounded on
    the vast feasible solutions revealed by the analysis of network representation. Despite of its
    random initialization nature, the proposed network learning shows relatively good
    prediction performance compared with the conventional gradient based learning.

    \item Network representation: The capability of network representation has been shown
    to be dependent on the number of data
    samples and the number of output vectors. It has been discovered that the warping effect of the nonlinear activation function contributes to
    the well-posedness of the problem.

    \item Network generalization: Capitalized on the minimum norm property according to
    Lemma~\ref{lemma_LS} and Lemma~\ref{lemma_LS_matrix}, no explicit regularization
    has been incorporated into the learning. The extensive experimental results
    provide relevant support towards this proposition.

\end{itemize}

To summarize the numerical evaluations, the following observations are made:
\begin{itemize}
    \item Practical feasibility: The numerical feasibility of estimating the weights of the multilayer
    network without gradient descent has been verified by utilizing a randomly initialized
    estimation. Essentially, the estimation process has been expressed in closed-from
    with neither gradient derivation nor iterative search.

    \item Prediction accuracy: Utilizing the logit activation function,
    the results of numerical prediction
    based on the stratified cross-validation show good generalization performance
    for the proposed kernel-and-range estimation comparing with
    that of the error gradient based learning. It is not surprising to observe that
    such learning performance does not depend on explicit regularization.

    \item Training CPU time: Due to the analytical and non-iterative estimation, the
    training speed for the proposed method is over 200 times faster than the competing
    gradient-based learning when the same number of hidden nodes has been chosen for the
    compared algorithms.

    \item Impact of deeper structure: According to the results based on deeper networks with
    an exponentially increasing node size towards the input, the prediction performance does not seem to
    vary significantly for deeper networks.
\end{itemize}

\section{Conclusion} \label{sec_conclusion}

We have established that the solution obtained from the kernel and range space
manipulation is equivalent to that obtained by the least squares error estimation. This
paved the ground for a gradient free parametric estimation when the system can be
expressed in linear matrix form. The fully-connected feedforward neural network is a
representative example for such a system. Based on the proposed kernel and range space
manipulation, solving the weights of the fully-connected feedforward network is found to
boil down to solving a cross-coupling equation. By initializing the weights in the hidden
layers, the system can be decoupled with hidden weights obtained in an analytical form.
In view of network representation, it has been found that the minimum number of
parameters is determined by the number of layers and the number of outputs, thanks to the
favorable warping effect of the nonlinear monotonic activations. The numerical evaluation
on real-world benchmark data sets validated the feasibility of the formulation. This
opens up the vast possibilities to expand the research direction.

\bibliographystyle{c:/ibb/tex1/IEEEtran2003}

\end{document}